\documentclass{article}

    \usepackage[nonatbib,final]{neurips_2020}

    \author{Adil Salim \qquad  Peter Richt\'{a}rik   \\
	King Abdullah University of Science and Technology, Thuwal, Saudi Arabia
}

\usepackage{url}            % simple URL typesetting
\usepackage{booktabs}       % professional-quality tables
\usepackage{amsfonts}       % blackboard math symbols
\usepackage{nicefrac}       % compact symbols for 1/2, etc.
\usepackage{microtype}      % microtypography

\usepackage{subcaption}

\usepackage{graphicx}

 \usepackage[english]{babel} 
  \usepackage{times}		
\usepackage[T1]{fontenc}
\usepackage[utf8]{inputenc} 
\DeclareUnicodeCharacter{00A0}{~}
\usepackage{amssymb,amsmath,amscd,amsfonts,amsthm,bbm,mathrsfs,yhmath}
\usepackage[shortlabels]{enumitem}
\usepackage{algorithm}
\usepackage{algpseudocode}
\usepackage{hyperref}
\usepackage[sort,nocompress]{cite}
\usepackage[textwidth=2cm, textsize=footnotesize]{todonotes}  
\newcommand{\cB}{{\mathcal B}} 

\newcommand{\cF}{{\mathcal F}}

\newcommand{\cP}{{\mathcal P}} 
\newcommand{\cO}{{\mathcal O}} 
\newcommand{\cE}{{\mathcal E}} 
\newcommand{\cH}{{\mathcal H}} 
 
\newcommand{\Lag}{{\mathscr L}} 
\newcommand{\DG}{{\mathscr D}} 
\newcommand{\mcF}{{\mathscr F}} 
\newcommand{\mcB}{{\mathscr B}} 
\newcommand{\mcG}{{\mathscr G}}

\DeclareMathOperator{\Leb}{Leb}
\DeclareMathOperator{\prox}{prox}

\DeclareMathOperator{\dom}{dom}
\DeclareMathOperator{\interior}{int}

\newcommand{\eqdef}{:=}

\newcommand{\bR}{{\mathbb R}} 
\newcommand{\bN}{{\mathbb N}} 
\newcommand{\bP}{{\mathbb P}} 
\newcommand{\indic}{{\mathbf 1}} 
\newcommand{\bE}{{\mathbb E}} 
\newcommand{\bV}{{\mathbb V}} 
\newcommand{\sX}{{\mathsf X}} 
 
\newcommand{\sZ}{{\mathsf Z}}

\newcommand{\partialb}{{\boldsymbol{\partial}}}

\newcommand{\cL}{{{\mathcal L}}}

\newcommand{\KL}{\mathop{\mathrm{KL}}\nolimits}

\newcommand{\tr}{\mathop{\mathrm{tr}}\nolimits}
\newcommand{\ps}[1]{\langle #1 \rangle}
\newcommand{\supp}{\mathop{\mathrm{supp}}\nolimits}

\theoremstyle{definition}
\newtheorem{theorem}{Theorem}
\newtheorem{lemma}[theorem]{Lemma}
\newtheorem{corollary}[theorem]{Corollary}

\newtheorem{remark}{Remark}

\newtheorem{assumption}{Assumption}

\DeclareMathOperator*{\argmin}{arg\,min}

\title{Primal Dual Interpretation of the Proximal Stochastic Gradient Langevin Algorithm}

\begin{document}

\maketitle

\begin{abstract}
We consider the task of sampling with respect to a log concave probability distribution. The potential of the target distribution is assumed to be composite, \textit{i.e.}, written as the sum of a smooth convex term, and a nonsmooth convex term possibly taking infinite values. The target distribution can be seen as a minimizer of the Kullback-Leibler divergence defined on the Wasserstein space (\textit{i.e.}, the space of probability measures). In the first part of this paper, we establish a strong duality result for this minimization problem. In the second part of this paper, we use the duality gap arising from the first part to study the complexity of the Proximal Stochastic Gradient Langevin Algorithm (PSGLA), which can be seen as a generalization of the Projected Langevin Algorithm. Our approach relies on viewing PSGLA as a primal dual algorithm and covers many cases where the target distribution is not fully supported. In particular, we show that if the potential is strongly convex, the complexity of PSGLA is $\cO(1/\varepsilon^2)$ in terms of the 2-Wasserstein distance. In contrast, the complexity of the Projected Langevin Algorithm is $\cO(1/\varepsilon^{12})$ in terms of total variation when the potential is convex.
\end{abstract}

%\asnote{Next step: distributed Langevin/ prox stochastic}

\section{Introduction}
% Consider a convex lower semicontinuous and proper function $V : \bR^d \to (-\infty,+\infty]$. Denote $\dom(V) = \{x \in \bR^d, V(x) < \infty\}$ the domain of $V$ and $\sX$ the affine space generated by $\dom(V)$. For simplicity, we assume that $0 \in \dom(V)$, therefore $0 \in \sX$ and $\sX$ is an Euclidean space. From now on, every functions will be seen as functions defined on the Euclidean space $\sX$ and we take $\sX$ as the reference space. For example, $V \in \Gamma_0(\sX)$. Moreover, the interior of $\dom(V)$ (w.r.t.\ $\sX$) denoted $\interior(\dom(V)$ is not empty~\cite[Th. 6.2]{rockafellar1970convex}.
%\asnote{define $\propto$, Fenchel transform, smooth, define $\sim$}
Sampling from a target distribution is a fundamental task in machine learning. Consider the Euclidean space $\sX = \bR^d$ and a convex function $V : \sX \to (-\infty,+\infty]$. Assuming that $\exp(-V)$ has a positive finite integral w.r.t.\ the Lebesgue measure $\Leb$, we consider the task of sampling from the distribution $\mu^\star$ whose density is proportional to $\exp(-V(x))$ (we shall write $\mu^\star \propto \exp(-V)$).

If $V$ is smooth, Langevin algorithm produces a sequence of iterates $(x^k)$ asymptotically distributed according to a distribution close to $\mu^\star$. Langevin algorithm performs iterations of the form
\begin{equation}
    \label{eq:Langevin-vanilla}
    x^{k+1} = x^k - \gamma \nabla V(x^k) + \sqrt{2\gamma} W^{k+1},
\end{equation}
where $\gamma >0$ and $(W^k)_k$ is a sequence of i.i.d.\ standard Gaussian vectors in $\sX$. Each iteration of~\eqref{eq:Langevin-vanilla} can be seen as a gradient descent step for $V$, where the gradient of $V$ is perturbed by a Gaussian vector. Hence, the iterations of Langevin algorithm look like those of the stochastic gradient algorithm; however the noise in Langevin algorithm is scaled by $\sqrt{\gamma}$ instead of $\gamma$. Nonasymptotic bounds for Langevin algorithm have been established in~\cite{dalalyan2017theoretical,durmus2017nonasymptotic}. Moreover, Langevin algorithm can be interpreted as an inexact gradient descent method to minimize the Kullback-Leibler (KL) divergence w.r.t.\ $\mu^\star$ in the space of probability measures~\cite{cheng2018convergence,durmus2018analysis,bernton2018langevin,wibisono2018sampling,ma2019there,ambrosio2008gradient}. 

In many applications, the function $V$ is naturally written as the sum of a smooth and a nonsmooth term. In Bayesian statistics for example, $\mu^\star$ typically represents some posterior distribution. In this case, $V$ is the sum of the $\log$-likelihood (which is itself a sum over the data points) and the possibly nonsmooth potential of the prior distribution~\cite{welling2011bayesian,durmus2018efficient,durmus2018analysis}, which plays the role of a regularizer. In some other applications in Bayesian learning, the support of $\mu^\star$ is not the whole space $\sX$~\cite{bubeck2018sampling,brosse2017sampling} (\textit{i.e.}, $V$ can take the value $+\infty$). In order to cover these applications, we consider the case where $V$ is written as
\begin{equation}
\label{eq:V}
V(x) \eqdef \bE_\xi(f(x,\xi)) + G(x),
\end{equation}
where $\xi$ is a random variable, $f(\cdot,s) : \sX \to \bR$ for every $s \in \Xi$, $F(x) = \bE_\xi(f(x,\xi))$ is smooth and convex and $G : \sX \to (-\infty,+\infty]$ is nonsmooth and convex. We assume to have access to the stochastic gradient $\nabla_x f(x,\xi)$ (where $\xi$ is a random variable with values in $\Xi$) and to the proximity operator $\prox_{\gamma G}$ of $G$. The template~\eqref{eq:V} covers many log concave densities\cite{Chaux2007,durmus2018efficient,durmus2018analysis,bubeck2018sampling}. In optimization, the minimization of $V$ can be efficiently tackled by the proximal stochastic gradient algorithm \cite{atc-for-mou-17}. Inspired by this optimization algorithm, the Proximal Stochastic Gradient Langevin Algorithm (PSGLA)~\cite{durmus2018analysis} is the method performing proximal stochastic gradient Langevin steps of the form
\begin{equation}
    \label{eq:psgla}
    x^{k+1} = \prox_{\gamma G}\left(x^k - \gamma \nabla_x f(x^k,\xi^{k+1}) + \sqrt{2\gamma} W^{k+1}\right),
\end{equation}
where $\gamma > 0$, $(W^k)$ is a sequence of i.i.d.\ standard Gaussian random vectors in $\sX$, and $(\xi^k)$ is a sequence of i.i.d.\ copies of $\xi$. Remarkably, the iterates $x^k$ of PSGLA remain in the domain of $G$, \textit{i.e.}, the support of $\mu^\star$, a property that is useful in many contexts. When $G$ is Lipschitz continuous, the support of $\mu^\star$ is $\sX$ and PSGLA can be interpreted as an inexact proximal gradient descent method for minimizing KL, with convergence rates  proven in terms of the KL divergence \cite{durmus2018analysis}. However, for  general $G$, the KL divergence can take infinite values along PSGLA. Therefore, a new approach is needed.  

 \subsection{Related works}
\label{sec:relat}
%\asnote{Talk about the dependence in the dimension $d$: requires further assumptions on $V$ and to start from a minimizer of $V$. Cf control of $\sigma_F$ by $d$ in~\cite{durmus2018analysis}}
\paragraph{First, various instances of the PSGLA algorithm have already been considered.} The only instance allowing $G(x)$ to be infinite (\textit{i.e.}, the support of $\mu^\star$  not to be $\sX$) is the Projected Langevin Algorithm~\cite{bubeck2018sampling}, which corresponds to our setting in the special case with $G = \iota_C$ (\textit{i.e.}, the indicator function of a convex body\footnote{A convex body is a compact convex set with a nonempty interior.} $C$), and $\nabla f(\cdot,s) \equiv \nabla F$ for every $s$ (\textit{i.e.}, the full gradient of $F$). In this case, $\prox_{\gamma G}$ is the orthogonal projection onto $C$ and $\mu^\star$ is supported by $C$. Bubeck et al~\cite{bubeck2018sampling} provide complexity results in terms of sufficient number of iterations to achieve $\varepsilon$ accuracy in terms of the Total Variation between the target distribution $\mu^\star$ and the current iterate distribution. Assuming that $F$ is convex and smooth, the complexity of the Projected Langevin Algorithm is $\cO(1/\varepsilon^{12})$\footnote{Our big O notation ignores logarithm factors.}, and if $F \equiv 0$, the complexity is improved to $\cO(1/\varepsilon^8)$. 

Other instances of PSGLA were proposed in the case where $G$ is Lipschitz continuous or smooth (and hence finite). Wibisono~\cite{wibisono2018sampling} considered the case with $F = G$ and $\nabla f(\cdot,s) \equiv \nabla F$, proposing the  Symmetrized Langevin Algorithm (SLA),  and showed that the current iterate distribution converges linearly in Wasserstein distance to the invariant measure of the SLA, if $F$ is strongly convex and smooth. %This result is not enough to provide a complexity result of the SLA in Wasserstein distance w.r.t the target distribution. The complexity of SLA in terms of Wasserstein distance is left as an open question~\cite[Question 1]{wibisono2018sampling}. 
Durmus et al~\cite{durmus2018analysis} considered the case where $G$ is Lipschitz continuous, and showed that the complexity of PSGLA is $\cO(1/\varepsilon^2)$ in terms of the KL divergence and $\cO(1/\varepsilon^4)$ in terms of the Total Variation distance if $F$ is convex and smooth. If $F$ is strongly convex, the complexity is  $\cO(1/\varepsilon^2)$ in Wasserstein distance and $\cO(1/\varepsilon)$ in KL divergence. Bernton~\cite{bernton2018langevin} studied a setting similar to~\cite{durmus2018analysis} and derived a similar result for the Proximal Langevin Algorithm (\textit{i.e.}, PSGLA without the gradient step) in the strongly convex case. The Proximal Langevin Algorithm was also studied in a recent paper of Wibisono~\cite{wibisono2019proximal}, where a rapid convergence result was proven in the case where $G$ is nonconvex but satisfies further smoothness and geometric assumptions. 

\paragraph{Second, the task of sampling w.r.t.\ $\mu^\star$, where $G$ is nonsmooth and possibly takes infinite values, using Langevin algorithm, has also been considered.} When $F$ is strongly convex and $G$ an indicator function of a bounded convex set, the existence of an algorithm achieving $\cO(1/\varepsilon^2)$ in Wasserstein and Total Variation distances was proven by Hsieh et al~\cite[Theorem 3]{hsieh2018mirrored}. However, an actual algorithm is only given in a specific, although nonconvex, case. Besides, MYULA (Moreau-Yosida Unadjusted Langevin Algorithm)~\cite{durmus2018efficient,brosse2017sampling} can tackle the task of sampling from $\mu^\star$ efficiently. MYULA is equivalent to Langevin algorithm~\eqref{eq:Langevin-vanilla} applied to sampling from $\mu^\lambda \propto \exp(-F-G^\lambda)$, where $G^\lambda$ is the Moreau-Yosida approximation of $G$~\cite{atchade2015moreau}. By choosing the smoothing parameter $\lambda >0$ appropriately, and making assumptions that allow to control the distance between $\mu
^\lambda$ and $\mu^\star$ (\textit{e.g.}, $G$ Lipschitz or $G = \iota_C$), complexity results for MYULA were established in~\cite{durmus2018efficient,brosse2017sampling}. For example, if $G$ is the indicator function of a convex body, Brosse et al~\cite{brosse2017sampling} show that the complexity of MYULA is $\cO(1/\varepsilon^{6})$ in terms of the Total Variation distance (resp.\ 1-Wasserstein distance) if $F$ is convex and smooth (resp., if $F$ is strongly convex and smooth), provided that the algorithm is initialized from a minimizer of $V$. Similarly to PSGLA, MYULA involves one proximal step and one gradient step per iteration. However, the support of the smoothed distribution $\mu^\lambda$ is always $\sX$ (even if $\mu^\star$ is not fully supported), and therefore the iterates of MYULA do not remain in the support of the target distribution $\mu^\star$, contrary to PSGLA. 

% Some of the complexity results of Langevin algorithms allowing $G$ to take infinite values are summarized in Table~\ref{tab:rate}.

% \begin{table}[t]
%   \centering
%   \caption{Complexity of PSGLA to achieve $\varepsilon$ accuracy.} 
%   \label{tab:rate}
%   \footnotesize
%   \begin{tabular}[h]{|c|c|c|c|c|}
%     \hline
%   Paper & Assumption & Criterion  & Rate & Comment\\
%   \hline
%   \cite{bubeck2018sampling} & $G = \iota_C$ & Total Variation & $\mO(1/\varepsilon^{12})$ & Instance of PSGLA \\
%   \hline
%   \cite{brosse2017sampling} & $F$ strongly convex, $G = \iota_C$ & 1-Wasserstein & $\mO(1/\varepsilon^6)$ & Iterates do not remain in $\supp(\mu^\star)$ \\
%     \hline
%   \cite{hsieh2018mirrored} & $F$ strongly convex, $G = \iota_C$ & 2-Wasserstein & $\mO(1/\varepsilon^2)$ & Existence result \\
%   \hline
%   Corollary~\ref{cor:rate} & $F$ strongly convex, $\mu^\star \in S_{\text{loc}}^{1,1}(\sX)$ & 2-Wasserstein & $\mO(1/\varepsilon^2)$ &  \\
%       \hline
%   Corollary~\ref{cor:rate} & $\mu^\star \in S_{\text{loc}}^{1,1}(\sX)$ & Duality gap & $\mO(1/\varepsilon^2)$ & Duality gap replaces KL which can be infinite \\
%   \hline
% \end{tabular}
% \end{table}

\paragraph{Finally, the task of sampling w.r.t.\ $\mu^\star$, where $V$ is not smooth but finite, has also been considered.} The Perturbed Langevin Algorithm proposed by Chatterji et al~\cite{chatterji2019langevin} allows to sample from $\mu^\star$ in the case when $G$ satisfies a weak form of smoothness (generalizing both Lipschitz continuity and smoothness) and without accessing its proximity operator. Finally, if $G$ is Lipschitz continuous, the Stochastic Proximal Langevin Algorithm proposed by Salim et al~\cite{salim2019stochastic}  and Schechtman et al\cite{schechtman2019passty} allows to sample from $\mu^\star$ using  cheap stochastic proximity operators only.

\subsection{Contributions}

In summary, PSGLA has complexity $\cO(1/\varepsilon^{2})$ in 2-Wasserstein distance if $F$ is strongly convex~\cite{durmus2018analysis} and $G$ is Lipschitz. The only instance of PSGLA allowing $G$ to be infinite is the Projected Langevin Algorithm. It has complexity $\cO(1/\varepsilon^{12})$ in Total Variation~\cite{bubeck2018sampling}\footnote{This result also holds if $F$ is not strongly convex.} and only applies to the case where $G$ is the indicator of a convex body. In the latter case, another Langevin algorithm called MYULA has complexity $\cO(1/\varepsilon^{6})$ in 1-Wasserstein distance~\cite{brosse2017sampling}, but allows the iterates to leave the support of $\mu^\star$. Besides, still in the case where $G$ is an indicator function, there exists a Langevin algorithm achieving $\cO(1/\varepsilon^2)$ rate in the Wasserstein distance~\cite{hsieh2018mirrored}. 
%\asnote{We cannot take $G$ to be an indicator (because of the Sobolev assumption)}

{\em In this paper, we consider other (\textit{i.e.}, new) cases where $G$ can take infinite values. More precisely, we consider a general nonsmooth convex function $G$ and we assume that $\exp(-V)$ has a mild Sobolev regularity. We develop new mathematical tools (\textit{e.g.}, a Lagrangian for the minimization of KL), that have their own interest, to obtain our complexity results.  Our main result is to show that, surprisingly, PSGLA still has the complexity $\cO(1/\varepsilon^2)$ in 2-Wasserstein distance if $F$ is strongly convex, although $G$ can take infinite values. We also show that, if $F$ is just convex, PSGLA has the complexity $\cO(1/\varepsilon^2)$ in terms of a newly defined duality gap, which can be seen as the notion that replaces KL, since KL can be infinite.}

 %Since Langevin algorithm cannot be applied in this case, another algorithm (MYULA~\cite{brosse2017sampling}) relies on smoothing the original problem and applying Langevin to it. This adds an extra bias to the algorithm and 

Our approach follows the line of works~\cite{cheng2018convergence,durmus2018analysis,wibisono2018sampling,bernton2018langevin,ma2019there,wibisono2019proximal,vempala2019rapid,rolland2020double} that formulate the task of sampling form $\mu^\star$ as the problem of minimizing the KL divergence w.r.t $\mu^\star$. 
%\asnote{MYULA as a benchmark}
% The main idea of this paper is as follows. In optimization, duality theory provides tools to study nonsmooth problems~\cite{rockafellar1970convex}. For example, convergence results of optimization algorithms applied to nonsmooth problems can be expressed in terms of duality gap~\cite{chambolle2011first,chambolle2016ergodic,yan2018new}. We adopt a similar strategy to analyze our algorithm~\ref{eq:LangevinLV}. 
In summary, our contributions are the following.

$\bullet$ In the first part of the paper, we  reformulate the task of sampling from $\mu^\star$ as the resolution of a monotone inclusion defined on the space of probability measures.  We subsequently  use this reformulation to  define a duality gap for the minimization of the KL divergence, and show that  strong duality holds.

$\bullet$ In the second part of this paper,  we use this reformulation to  represent PSGLA as a primal dual stochastic Forward Backward algorithm involving monotone operators.

$\bullet$ This new representation of PSGLA, along with the strong duality result from the first part, allows us to  prove new complexity results for PSGLA that extend and improve the state of the art.

$\bullet$ Finally, we conduct some numerical experiments for sampling from a distribution supported by a set of matrices (see appendix).

In the first part we combine tools from optimization duality~\cite{condatreview} and optimal transport~\cite{ambrosio2008gradient} and in the second part we combine tools from the analysis of the Langevin algorithm~\cite{durmus2018analysis}, and the analysis of primal dual optimization algorithms~\cite{drori2015simple,chambolle2016ergodic}.
% In this paper, we show that $\mu^\star$ is the solution to a saddle point problem and define a duality gap for the saddle point problem. Then, we show that strong duality holds for the saddle point problem and prove some properties of the newly defined duality gap. Using these results, we prove non asymptotic convergence rates for PSGLA~\eqref{eq:psgla} in terms of 2-Wasserstein distance and in terms of the duality gap. More precisely, we prove that the complexity of PSGLA is $\cO(d/\varepsilon^2)$ in terms of 2-Wasserstein distance if $F$ is strongly convex. We also show complexity results in terms of the duality gap ($\cO(d/\varepsilon^2)$ if $F$ is convex and $\cO(d/\varepsilon^2)$ if $F$ is strongly convex). If $G$ is smooth, we show another complexity result involving the Wasserstein distance w.r.t.\ some dual distribution to $\mu^\star$. Finally, we conduct some numerical experiments for sampling from a distribution whose support is a possibly unbounded subset of $\sX$.

The remainder is organized as follows. In Section~\ref{sec:background} we provide some background knowledge on convex analysis and optimal transport. In Section~\ref{sec:pd-opt} we develop a primal dual optimality theory for the task of sampling from $\mu^\star$. In Section~\ref{sec:alg} we give a new representation of PSGLA using monotone operators. We use it to state our main complexity result on PSGLA in Section~\ref{sec:res}. Numerical experiments and all proofs are postponed to the appendix. Therein, we also provide further intuitions on PSGLA, namely the connection between gradient descent and Langevin algorithm~\cite{durmus2018analysis} and the connection between primal dual optimization and our approach. Finally, an extension of PSGLA for handling a third (stochastic, Lipschitz continuous and proximable) term in the definition of the potential $V$~\eqref{eq:V} is provided at the end of the appendix.

\section{Background}
\label{sec:background}
Throughout this paper, we use the conventions $\exp(-\infty) = 0$ and $1/0 = +\infty$.
\subsection{Convex analysis}
In this section, we recall some facts from convex analysis. These facts will be used in the proofs without mention. For more details, the reader is referred to~\cite{bau17}.
\subsubsection{Convex optimization}
By  $\Gamma_0(\sX)$ we denote the set of proper, convex, lower semicontinuous functions $\sX \to (-\infty,+\infty]$. A function $F \in \Gamma_0(\sX)$ is $L$-smooth if $F$ is differentiable and its gradient $\nabla F$ is $L$-Lipschitz continuous. Consider $G \in \Gamma_0(\sX)$ and denote $\dom(G) \eqdef \{x \in \sX \;: \; G(x) < \infty\}$ its domain. Given $x \in \sX$, a subgradient of $G$ at $x$ is any vector $y \in \sX$ satisfying
\begin{equation}
    G(x) + \ps{y,x' - x} \leq G(x'),
\end{equation}
for every $x' \in \sX$. If the set $\partial G(x)$ of subgradients of $G$ at $x$ is not empty, then there exists a unique element of $\partial G(x)$ with minimal norm. This particular subgradient is denoted $\partial^0 G(x)$. The set valued map $\partial G(\cdot)$ is called the subdifferential. The proximity operator of $G$, denoted $\prox_{G}$, is defined by
\begin{equation}
    \label{eq:prox}
    \prox_{G}(x) \eqdef \argmin_{x' \in \sX} \left\{G(x') + \tfrac{1}{2}\|x-x'\|^2\right\}. 
\end{equation}
By $\iota_C(\cdot)$ we denote the indicator function of set $C$ given by $\iota_C(x) = 0$ if $x \in C$ and $\iota_C(x) =+\infty$ if $x \notin C$. If $G=\iota_C$, where $C$ is a closed convex set, then $\prox_{G}$ is the orthogonal projection onto $C$. Moreover, $\prox_{G}(x)$ is the only solution $x'$ to the inclusion $x \in x' + \partial G(x')$. The Fenchel transform of $G$ is the function $G^\ast \in \Gamma_0(\sX)$ defined by $G^\ast(y) \eqdef \sup_{x \in \sX} \left\{ \ps{y,x} - G(x) \right\}.$
Several properties relate $G$ to its Fenchel transform $G^\ast$. First, the Fenchel transform of $G^\ast$ is $G$. Then, the subdifferential $\partial G^\ast$ is characterized by the relation $x \in \partial G^\ast(y) \Leftrightarrow y \in \partial G(x).$ Finally, $G^\ast$ is $\lambda$-strongly convex if and only if $G$ is $1/\lambda$-smooth.

%Finally, $G^\ast$ is $\lambda$-strongly convex if and only if $G$ is $1/\lambda$-smooth.
%Moreover, Moreau's identity states that
%\begin{equation}
%\label{eq:moreau}
%\prox_{\gamma G}(x) = x - \gamma \prox_{G^\ast/\gamma}(x/\gamma),
%\end{equation}
%For every $x \in \sX$. 
\subsubsection{Maximal monotone operators}
A set valued function $A : \sX \rightrightarrows \sX$ is {\em monotone} if $\ps{y-y',x-x'} \geq 0$ whenever $y \in A(x)$ and $y' \in A(x')$. The {\em inverse} of $A$, denoted $A^{-1}$, is defined by the relation $x \in A^{-1}(y) \Leftrightarrow y \in A(x)$, and the {\em set of zeros} of $A$ is $Z(A) \eqdef A^{-1}(0)$. If $A$ is monotone, $A$ is {\em maximal} if its {\em resolvent}, \textit{i.e.}, the map $J_{A} : x \mapsto (I+A)^{-1}(x)$, is  single valued. If $G \in \Gamma_0(\sX)$, then $\partial G$ is a maximal monotone operator and $J_{\partial G } = \prox_{G}$. Moreover, $Z(\partial G) = \argmin G$ and $(\partial G)^{-1} = \partial G^\ast$. If $S$ is a skew symmetric matrix on $\sX$, the operator $x \mapsto S x$ is maximal monotone. Finally, the sum $\partial G + S$ is also a maximal monotone operator. Many problems in optimization can be cast as the problem of finding a zero $x$ of the sum of two maximal monotone operators $0 \in (A+B)(x)$ \cite{condatreview}. For instance, $Z(\nabla F + \partial G) = \argmin F+G$. To solve this problem, the Forward Backward algorithm is given by the iteration $x^{k+1} = J_{P^{-1}A}(x^k - P^{-1}B(x^k))$, where $P$ is a symmetric positive definite matrix ($P \in \bR^{d \times d}_{++}$),\footnote{The operators $P^{-1}A$ and $P^{-1}B$ are not monotone in general, however they are monotone under the inner product induced by $P$.} and $B$ is single valued. Using the definition of the resolvent, the Forward Backward algorithm can equivalently be written as 
\begin{equation}
    P(x^{k+1/2} - x^k) = -\gamma B(x^k), \quad P(x^{k+1} - x^{k+1/2}) \in -\gamma A(x^{k+1}).
\end{equation}

\subsection{Optimal transport}
In this section, we recall some facts from optimal transport theory. These facts will be used in the proofs without mention. For more details, the reader is referred to Ambrosio et al~\cite{ambrosio2008gradient}.
\subsubsection{Wasserstein distance}
By $\mcB(\sX)$ we denote the $\sigma$-field of Lesbesgue measurable subsets of $\sX$, and by $\cP_2(\sX)$ the set of probability measures $\mu$ over $(\sX,\mcB(\sX))$ with finite second moment $\int \|x\|^2 d\mu(x) < \infty$. Denote $\supp(\mu)$ the support of $\mu$. The identity map $I$ belongs to the Hilbert space $L^2(\mu;\sX)$ of $\mu$-square integrable random vectors in $\sX$. We denote $\ps{\cdot,\cdot}_\mu$ (resp.\ $\|\cdot\|_{\mu}$) the inner product (resp.\ the norm) in this space.
Given $T : \sX \to \sZ$, where $\sZ$ is some Euclidean space, the {\em pushforward measure} of $\mu$ by $T$, also called the {\em image measure}, is defined by $T \# \mu (A) \eqdef \mu(T^{-1}(A))$ for every $A \in \mcB(\sZ)$. Consider $\mu,\nu \in \cP_2(\sX)$. A {\em coupling} $\upsilon$ between $\mu$ and $\nu$ (we shall write $\upsilon \in \Gamma(\mu,\nu)$) is a probability measure over $(\sX^2,\mcB(\sX^2))$ such that $x^\star \# \upsilon = \mu$, where $x^\star : (x,y) \mapsto x$, and $y^\star \# \upsilon = \nu$, where $y^\star : (x,y) \mapsto y$. In other words, $(X,Y)$ is a random variable such that the distribution of $X$ is $\mu$ (we shall write $X \sim \mu$) and $Y \sim \nu$ if and only if the distribution of $(X,Y)$ is a coupling. The (2-){\em Wasserstein distance} is then defined by
\begin{equation}
\label{eq:wass}
    W^2(\mu,\nu) \eqdef \inf_{\upsilon \in \Gamma(\mu,\nu)} \int \|x-y\|^2 d\upsilon(x,y).
\end{equation}
Let $\cP_2^r(\sX)$ be the set of elements $\mu \in \cP_2(\sX)$ such that $\mu$ is absolutely continuous w.r.t.\ $\Leb$ (we shall write $\mu \ll \Leb$). Brenier's theorem asserts that if $\mu \in \cP_2^r(\sX)$, then the $\inf$ defining $W^2(\mu,\nu)$ is actually a $\min$ achieved by a unique minimizer $\upsilon$. Moreover, there exists a uniquely determined $\mu$-almost everywhere (a.e.) map $T_{\mu}^{\nu} : \sX \to \sX$ such that $\upsilon = (I,T_{\mu}^{\nu}) \# \mu$, where $(I,T_{\mu}^{\nu}) : x \mapsto (x,T_{\mu}^{\nu}(x))$. In this case, $T_{\mu}^{\nu}$ is called the {\em optimal pushforward} from $\mu$ to $\nu$ and satisfies 
\begin{equation}
    W^2(\mu,\nu) = \int \|x-T_{\mu}^{\nu}(x)\|^2 d\mu(x).
\end{equation} 
%$\cP_2$, W, pushforward, brenier optimal coupling, Pushforward optimal+notation

\subsubsection{Geodesically convex functionals}

We shall consider several functionals defined on the space $\cP_2(\sX)$. For every $\mu \in \cP_2^r(\sX)$ with density denoted $\mu(x)$ w.r.t.\ $\Leb$, the {\em entropy} is defined by
\begin{equation}
\label{eq:entropy}
    \cH(\mu) \eqdef \int \log(\mu(x)) d\mu(x),
\end{equation}
and if $\mu \notin \cP_2^r(\sX)$, then $\cH(\mu) \eqdef +\infty$.
Given $V \in \Gamma_0(\sX)$, the {\em potential energy} is defined for every $\mu \in \cP_2(\sX)$ by 
\begin{equation}
\label{eq:potential}
    \cE_V(\mu) \eqdef \int V(x)d\mu(x).
\end{equation}
%\asnote{restart here. $\nu$ or $\mu'$?}
Finally, if $\mu' \in \cP_2(\sX)$ such that $\mu \ll \mu'$, the {\em Kullback-Leibler (KL) divergence} is defined by
\begin{equation}
\label{eq:kl}
    \KL(\mu|\mu') \eqdef \int \log\left(\tfrac{d\mu}{d\mu'}(x)\right) d\mu(x),
\end{equation}
where $\frac{d\mu}{d\mu'}$ denotes the density of $\mu$ w.r.t.\ $\mu'$, and $\KL(\mu|\mu') \eqdef +\infty$ if $\mu$ is not absolutely continuous w.r.t.\ $\mu'$.
The functionals $\cH$, $\cE_V$ and $\KL(\cdot|\mu^\star)$ satisfy a form of convexity over $\cP_2(\sX)$ called {\em geodesic convexity}. If $\cF : \cP_2(\sX) \to (-\infty,+\infty]$ is geodesically convex, then for every $\mu \in \cP_2^r(\sX)$, $\mu'\in\cP_2(\sX)$, and $\alpha \in [0,1]$, $\cF\left((\alpha T_{\mu}^{\mu'} + (1-\alpha) I)\#\mu\right) \leq \alpha \cF(\mu') +(1-\alpha)\cF(\mu).$
Given $\mu \in \cP_2^r(\sX)$, a {\em (Wasserstein) subgradient} of $\cF$ at $\mu$ is a random variable $Y \in L^2(\mu;\sX)$ such that for every $\mu' \in \cP_2(\sX)$, %\asnote{define monotonicity}
\begin{equation}
    \label{eq:subdiff}
    \cF(\mu) + \ps{Y,T_{\mu}^{\mu'} - I}_{\mu} \leq \cF(\mu').
\end{equation}
Moreover, if $Y'$ is a subgradient of $\cF$ at $\mu'$, then the following monotonicity property holds
\begin{equation}
    \label{eq:wass-monot}
    \ps{Y'\circ T_{\mu}^{\mu'} - Y,T_{\mu}^{\mu'} - I}_{\mu} \geq 0.
\end{equation}
If the set $\partialb \cF(\mu) \subset L^2(\mu;\sX)$ of subgradients of $\cF$ at $\mu$ is not empty, then there exists a unique element of $\partialb \cF(\mu)$ with minimal norm. This particular subgradient is denoted $\partialb^0 \cF(\mu)$. However, the set $\partialb \cF(\mu)$ might be empty. A typical condition for nonemptiness requires the density $\mu(x)$ to have some Sobolev regularity. For every open set $\Omega \subset \sX$, we denote $S^{1,1}(\Omega)$ the Sobolev space of $\Leb$-integrable functions $u : \Omega \to \bR$ admitting a $\Leb$-integrable weak gradient $\nabla u : \Omega \to \sX$. We say that $u \in S_{\text{loc}}^{1,1}(\Omega)$ if $u \in S^{1,1}(K)$ for every bounded open set $K \subset \Omega$. Obviously, $S^{1,1}(\Omega) \subset S_{\text{loc}}^{1,1}(\Omega)$.

\subsection{Assumptions on $F$ and $G$}
Consider $F : \sX \to \bR$ and $G : \sX \to (-\infty,+\infty]$. We make the following assumptions.
\begin{assumption}
\label{as:smooth}
The function $F$ is convex and $L$-smooth. Moreover, $G \in \Gamma_0(\sX)$. 
\end{assumption}
Note that $V \eqdef F+G \in \Gamma_0(\sX)$. We denote $\lambda_F$ (resp.\ $\lambda_{G^\ast}$) the strong convexity parameter of $F$ (resp.\ $G^\ast$), equal to zero if $F$ (resp.\ $G^\ast$) is not strongly convex.
\begin{assumption}
\label{as:int}
The integral $\int \exp(-V)d\Leb$ is positive and finite.
\end{assumption}
Assumption~\ref{as:int} is needed to define the target distribution $\mu^\star \propto \exp(-V)$, and implies that $\interior(D) \neq \emptyset$, where $D \eqdef \dom(V)$.
% \begin{assumption}
% \label{as:boundary}
% The set $D\setminus\interior(D)$ is $\Leb$-negligible.\asnote{We don't need this asumption thanks to~\cite[Eq. 10.4.57]{ambrosio2008gradient}}
% \end{assumption}
\begin{lemma}
\label{lem:Vint}
If Assumptions~\ref{as:smooth} and~\ref{as:int} hold, then
$$\int |V(x)|\exp(-V(x))dx < \infty,\quad \text{and} \quad \int \|x\|^2\exp(-V(x))dx < \infty.$$
\end{lemma}
Lemma~\ref{lem:Vint} implies that $\mu^\star \in \cP_2(\sX)$ and using Assumption~\ref{as:smooth}, $\|\nabla F\| \in L^2(\mu^\star;\bR)$. Since $G \in \Gamma_0(\sX)$, $G$ is differentiable $\Leb$-a.e.\ (almost everywhere) on $\interior(D)$, see~\cite[Theorem 25.5]{rockafellar1970convex}.
\begin{assumption}
\label{as:intgrad}
 The integral $\int_{\interior(D)} \|\nabla G\|^2\exp(-V)d\Leb$ is finite.
\end{assumption}
Assumption~\ref{as:intgrad} is equivalent to requiring $\|\nabla G\| \in L^2(\mu^\star; \bR)$, see below.
Moreover, we assume the following regularity property for the function $\exp(-V)$. 
\begin{assumption}
\label{as:sobolev}
The function $\exp(-V)$ belongs to the space $S_{\text{loc}}^{1,1}(\sX)$.
\end{assumption}
Assumption~\ref{as:sobolev} is a necessary condition for $\partialb \cH(\mu^\star) \neq \emptyset$, see below. This assumption precludes $\mu^\star$ from being a uniform distribution. However, Assumption~\ref{as:sobolev} is quite general, \textit{e.g.}, $\exp(-V)$ need not be continuous or positive (see the numerical experiment section). Finally, we assume that the stochastic gradients of $F$ have a bounded variance. Consider an abstract measurable space $(\Xi,\mcG)$, and a random variable $\xi$ with values in $(\Xi,\mcG)$.
\begin{assumption}
\label{as:bound-var}
For every $x \in \sX$, $f(x,\xi)$ is integrable and $F(x) = \bE_\xi(f(x,\xi))$. Moreover, there exists $\sigma_F \geq 0$ such that for every $x \in \sX$, $\bV_\xi(\|\nabla f(x,\xi)\|) \leq \sigma_F^2$, where $\bV$ denotes the variance.
\end{assumption}
The last assumption implies that the stochastic gradients are unbiased: $\bE_\xi(\nabla f(x,\xi)) = \nabla F(x)$ for every $x \in \sX$.

\section{Primal dual optimality in Wasserstein space}
\label{sec:pd-opt}
Let $\cF : \cP_2(\sX) \to (-\infty,+\infty]$ be defined by
\begin{equation}
\label{eq:def-cF}
\cF(\mu) \eqdef \cH(\mu) + \cE_V(\mu) = \cH(\mu) + \cE_F(\mu) + \cE_G(\mu).
\end{equation}
Using Lemma~\ref{lem:Vint}, $\cH(\mu^\star)$ and $\cE_V(\mu^\star)$ are finite real numbers. Moreover, using~\cite[Lemma 1.b]{durmus2018analysis}, for every $\mu \in \cP_2(\sX)$ such that $\cE_V(\mu) < \infty$, we have the identity 
\begin{equation}
\label{eq:KL-F}
    \cF(\mu) - \cF(\mu^\star) = \KL(\mu|\mu^\star).
\end{equation}
Equation~\eqref{eq:KL-F} says that $\mu^\star$ is the unique minimizer of $\cF$: $\mu^\star = \argmin \cF.$
\subsection{Subdifferential calculus}
The following result is a consequence of~\cite[Theorem 10.4.13]{ambrosio2008gradient}.

\begin{theorem}
\label{th:subdiff}
Let $\mu \propto \rho$ be an element of $\dom(\cF)$. Then, $\supp(\mu) \subset \overline{D}$ and $\mu(\overline{D}\setminus \interior(D)) = 0$. Moreover, $\partialb \cF(\mu) \neq \emptyset$ if and only if $\rho \in S_{\text{loc}}^{1,1}(\interior(D))$ and there exists $w \in L^2(\mu)$ such that 
\begin{equation}
    w(x) \rho(x) = \nabla \rho(x) + \rho(x) \nabla V(x),
\end{equation}
for $\mu$-a.e.\ $x$. In this case, $w = \partialb^0 \cF(\mu)$.
\end{theorem}
If Assumptions~\ref{as:smooth} and~\ref{as:int} hold, then $\cF(\mu^\star) < \infty$ using Lemma~\ref{lem:Vint}. Then, Theorem~\ref{th:subdiff} implies that $\mu^\star(\interior(D)) = 1$. Therefore, using~\cite[Theorem 25.5]{rockafellar1970convex}, $G$ and $V$ are $\mu^\star$-a.s.\ differentiable.

Moreover, applying Theorem~\ref{th:subdiff} with $V \equiv 0$, we can replace $\cF$ by $\cH$ and $D$ by $\sX$. We obtain that $\partialb \cH(\mu) \neq \emptyset$ if and only if $\rho \in S_{\text{loc}}^{1,1}(\sX)$ and $w \rho = \nabla \rho$ for some $w \in L^2(\mu;\sX)$. Now, we set $\mu = \mu^\star$ and $\rho = \exp(-V)$. Using Assumption~\ref{as:sobolev} and $w = -\nabla V$, we obtain that $\partialb^0 \cH(\mu^\star) = -\nabla V$ $\mu^\star$-a.e. Therefore, using that $\nabla G$ is well defined $\mu^\star$-a.e., $\mu^\star$ satisfies
\begin{equation}
\label{eq:FOC}
    0 = \nabla F(x) + \partialb^0 \cH(\mu^\star)(x) + \nabla G(x), \text{ for } \mu^\star-\text{a.e. } x.
\end{equation}
Equation~\eqref{eq:FOC} can be seen as the first order optimality conditions associated with the minimization of the functional $\cF$. Consider the "dual" variable $Y^\star : x \mapsto \nabla G(x)$ defined $\mu^\star$ a.e. Using Assumption~\ref{as:intgrad} and $\mu^\star(\interior(D)) = 1$, $Y^\star \in L^2(\mu^\star;\sX)$. We can express the first order optimality condition~\eqref{eq:FOC} as $0 = \nabla F(x) + \partialb^0 \cH(\mu^\star)(x) + Y^\star(x)$, $\mu^\star$ a.e. Besides, $Y^\star(x) \in \partial G(x)$, therefore $0 \in - x + \partial G^\ast(Y^\star(x))$ using $\partial G^\ast = (\partial G)^{-1}.$ Denote $\nu^\star \eqdef Y^\star \# \mu^\star \in \cP_2(\sX)$ and $\pi^\star \eqdef (I,Y^\star) \# \mu^\star \in \cP_2(\sX^2)$. The relationship between $\mu^\star$ and $Y^\star$ can be summarized as
% Finally, $\mu^\star$ satisfies
% \begin{equation}
% \label{eq:saddleW}
% \begin{bmatrix} 0 \\ 0\end{bmatrix} \in \begin{bmatrix} \nabla F(x) + \partialb^0 \cH(\mu^\star)(x) & + Y^\star(x) \\ - x& + \partial G^\ast(Y^\star(x))\end{bmatrix} \text{ for } \mu^\star \text{ a.e. } x.
% \end{equation}
\begin{equation}
\label{eq:saddleW2}
\begin{bmatrix} 0 \\ 0\end{bmatrix} \in \begin{bmatrix} \nabla F(x) + \partialb^0 \cH(\mu^\star)(x) & + y \\ - x& + \partial G^\ast(y)\end{bmatrix} \text{ for } \pi^\star \text{ a.e. } (x,y).
\end{equation}

%\asnote{Introduire $\pi$ mais ne pas encore parler de monotonie}

% The r.h.s. of Equation~\eqref{eq:saddleW} can be checked to satisfy a monotonicity property similar to~\cite[Eq. 10.1.8]{ambrosio2008gradient}. 
In the sequel, we fix the probability space $(\Omega,\mcF,\bP) = (\sX^2,\cB(\sX^2),\pi^\star)$, denote $\bE$ the mathematical expectation and $L^2$ the space $L^2(\Omega,\mcF,\bP;\sX)$. The expression "almost surely" (a.s.) will be understood w.r.t.\ $\bP$. Recall that $x^\star$ is the map $(x,y) \mapsto x$ and $y^\star : (x,y) \mapsto y$. Using Assumption~\ref{as:intgrad}, $x^\star, y^\star \in L^2$, $x^\star \sim \mu^\star$, $y^\star \sim \nu^\star$, $(x^\star,y^\star) \sim \pi^\star$ and $y^\star = \nabla G(x^\star)$ a.s. 

\subsection{Lagrangian function and duality gap}

We introduce the following Lagrangian function defined for every $\mu \in \cP_2(\sX)$ and $y \in L^2$ by
\begin{equation}
    \label{eq:LagrangianW}
    \Lag(\mu,y) \eqdef \cE_F(\mu) + \cH(\mu) - \cE_{G^\ast}(\nu) + \bE \ps{x, y},
\end{equation}
where $x = T_{\mu^\star}^\mu(x^\star)$. This Lagrangian is similar to the one used in Euclidean optimization; see the appendix. We also define the duality gap by
\begin{equation}
    \label{eq:dualitygap}
    \DG(\mu,y) \eqdef \Lag(\mu,y^\star) - \Lag(\mu^\star,y).
\end{equation}
The next theorem, which is of independent interest, can be interpreted as a strong duality result for the Lagrangian function $\Lag$, see~\cite[Lemma 36.2]{rockafellar1970convex}.
\begin{theorem}[Strong duality]
\label{lem:DGbregman}
Let Assumptions~\ref{as:smooth}--\ref{as:sobolev} hold true. Then, for every $\mu \in \cP_2(\sX), y \in L^2$, $\DG(\mu,y) \geq 0$ and $\Lag(\mu,y) \leq \cF(\mu)$. Moreover, $(\mu^\star,y^\star)$ is a saddle point of $\Lag$ with saddle value $\cF(\mu^\star)$, \textit{i.e.},
\begin{equation}
\label{eq:minmax}
\Lag(\mu^\star, y) \leq \cF(\mu^\star) = \Lag(\mu^\star,y^\star) \leq \Lag(\mu,y^\star).
\end{equation}
Finally, $\Lag(\mu^\star, y) = \cF(\mu^\star)$ if and only if $y = y^\star$, and, if $F$ is strictly convex, $\cF(\mu^\star) = \Lag(\mu,y^\star)$ if and only if $\mu = \mu^\star$.
% where, recalling from~\eqref{eq:saddleW} that $x^\star \in \partial G^\ast(y^\star)$, we denote
% $$
% D_{G^\ast}(y,y^\star) \eqdef G^\ast(y) - G^\ast(y^\star) - \ps{x^\star,y - y^\star}.
% $$
\end{theorem}
The proof of Theorem~\ref{lem:DGbregman} relies on using~\eqref{eq:saddleW2} to write the duality gap as the sum of the Bregman divergences of $F$, $G^\ast$ and $\cH$. We shall use the nonnegativity of the duality gap to derive convergence bounds for PSGLA. 

\section{Forward Backward representation of PSGLA}
\label{sec:alg}
In this section, we present our viewpoint on PSGLA~\eqref{eq:psgla}. More precisely, we represent PSGLA as a (stochastic) Forward Backward algorithm involving (stochastic) monotone operators which are not necessarily subdifferentials. %\asnote{ Let us remark thatin order to understand the results of the later sections Corollary 11 suffices; the remainder of thissection may thus be skipped at first reading.}
 
\paragraph{Intuition.} Let $\pi \in \cP_2(\sX^2)$ and consider $A, B(\pi) \in L^2(\pi;\sX^2)$ the set valued maps
\begin{equation}
A : (x,y) \mapsto \begin{bmatrix} &   y \\ - x& +\partial G^\ast(y)\end{bmatrix}, \quad B(\pi) : (x,y) \mapsto \begin{bmatrix} \nabla F(x) + \partialb \cH(\mu)(x)\\ 0\end{bmatrix},
\end{equation}
where $\mu = x^\star \# \pi$. The maps $\pi \mapsto A$ and $\pi \mapsto B(\pi)$ satisfy a monotonicity property similar to~\eqref{eq:wass-monot} (note that $A$ is a maximal monotone operator as the sum of $S : (x,y) \mapsto (y,-x)$ and the subdifferential of the $\Gamma_0(\sX^2)$ function $(x,y) \mapsto G^\ast(y)$). Inclusion~\eqref{eq:saddleW2} can be rewritten as 
\begin{equation}
\label{eq:zero-op-monot}
    0 \in \left(A+B(\pi^\star)\right)(x,y), \text{ for } \pi^\star \text{ a.e. } (x,y).
\end{equation} 
\paragraph{Rigorous Forward Backward representation.} The ``monotone'' inclusion~\eqref{eq:zero-op-monot} intuitively suggests the following stochastic Forward Backward algorithm for obtaining samples from $\pi^\star$ (and hence from $\mu^\star$ by marginalizing):
\begin{align}
    P\begin{bmatrix} x^{k+1/2} - x^{k}\\ y^{k+1/2} - y^{k}\end{bmatrix} &= -\gamma \begin{bmatrix} \nabla f(x^k,\xi^{k+1}) - \sqrt{\frac{2}{\gamma}}W^{k+1}\\ 0 \end{bmatrix} \label{eq:forward}\\
    P\begin{bmatrix} x^{k+1} - x^{k+1/2}\\ y^{k+1} - y^{k+1/2}\end{bmatrix} &\in 
    -\gamma A(x^{k+1},y^{k+1})\label{eq:backward}.
\end{align}
Above,  $P \in \bR^{d \times d}_{++}$ is an appropriately chosen matrix. Indeed, Algorithm~\eqref{eq:forward}-\eqref{eq:backward} looks like a stochastic Forward-Backward algorithm~\cite{rosasco2016stochastic,com-pes-pafa16,bia-hac-16,bia-hac-sal-jca17} where the gradient is perturbed by a Gaussian vector, as in the Langevin algorithm~\eqref{eq:Langevin-vanilla}. In Algorithm~\eqref{eq:forward}-\eqref{eq:backward}, we cannot set $P$ to be the identity map of $\sX^2$ because the inclusion~\eqref{eq:backward} is intractable in this case. 
We take $P : (x,y) \mapsto x$, \textit{i.e.}, with our notations, $P = x^\star$. Although the matrix $P$ is only semi-definite positive, the next lemma shows that Algorithm~\eqref{eq:forward}-\eqref{eq:backward} is still well defined. More precisely, the next lemma shows that $x^{k+1} = \prox_{\gamma G}(x^{k+1/2})$ (by taking $z = (x^{k+1/2},y^{k+1/2})$ in the lemma) and hence the resulting algorithm~\eqref{eq:forward}-\eqref{eq:backward} is PSGLA. Based on the representation~\eqref{eq:forward}-\eqref{eq:backward} of PSGLA, the next lemma also provides an important inequality used later in the proof of Theorem~\ref{th:evi}.
% \begin{equation}
%     \label{eq:P}
% P := \begin{bmatrix} I&  0 \\ 0 & 0 \end{bmatrix}.
% \end{equation}
\begin{lemma}
\label{lem:resolvent-LV}
Let $z = (x,y), z'=(x',y') \in \sX^2$. Then $P(z' - z) \in
    -\gamma A(z')$ if and only if $x' = \prox_{\gamma G}(x)$ and $y' = \prox_{G^\ast/\gamma}(x/\gamma)$.
Moreover, if $G \in \Gamma_0(\sX)$ is $1/\lambda_{G^\ast}$-smooth, then
\begin{align}
\label{eq:funda-lv}
   \|x' - x^\star\|^2 \leq& \|x - x^\star\|^2 -2\gamma\left(G^\ast(y') - G^\ast(y^\star) - \ps{y',x^\star}+ \ps{y^\star,x}\right)\nonumber\\ &-\gamma(\lambda_{G^\ast} + \gamma)\|y'-y^\star\|^2 + \gamma^2\|y^\star\|^2. 
\end{align}
\end{lemma}

\section{Main results}
\label{sec:res}
%\asnote{Convergence of dual variables (auxiliary result). Convergence en duality gap auxiliary result. Main result is Wasserstein distance result.}
% Give theorems, results and represnetation of the algorithm

We now provide our main result on PSGLA~\eqref{eq:psgla}. For $r \in \bN/2$, denote $\mu^r$ (resp. $\nu^r$) the distribution of $x^r$ (resp. $y^r$), defined in the previous section.

% \section{Analysis}

% Our algorithm~\ref{eq:LangevinLV} can be written 
% \begin{equation}
% \label{eq:pk}
%     x^{k+1/2} = x^k-\gamma \nabla F(x^k) +\sqrt{2\gamma}W^{k+1},
% \end{equation}
% \begin{equation}
% \label{eq:FB}
% (x^{k+1},y^{k+1}) = J_{\gamma P^{-1} A}(x^{k+1/2},y^k)
% \end{equation}

% where $J_{\gamma P^{-1} A}$ is defined in Lemma~\ref{lem:resolvent-LV}.
\begin{theorem}
\label{th:evi}
Let Assumptions~\ref{as:smooth},~\ref{as:int}, \ref{as:intgrad} and~\ref{as:bound-var} hold true. If $F$ is $\lambda_F$-strongly convex and $G$ is $1/\lambda_{G^\ast}$-smooth, then for every $\gamma \leq 1/L$,
\begin{align}
\label{eq:th}
    W^2(\mu^{k+1},\mu^\star)
    \leq& (1-\gamma\lambda_F)W^2(\mu^{k},\mu^\star)  -\gamma(\lambda_{G^\ast} + \gamma) W^2(\nu^{k+1},\nu^\star) \nonumber\\
    &-2\gamma \left(\Lag(\mu^{k+1/2},y^\star) - \Lag(\mu^\star,y_\star^{k+1})\right) + \gamma^2 C,
\end{align}
where $C \eqdef \int_{\interior(D)} \|\nabla G(x)\|^2 d\mu^{\star}(x) + 2 (L d+\sigma_F^2)$ and
$
y_{\star}^{k+1} \eqdef \prox_{G^\ast/\gamma}(x^{k+1/2}_{\star}/\gamma) \sim \nu^{k+1},
$ where $x^{k+1/2}_{\star} \eqdef T_{\mu^\star}^{\mu^{k+1/2}}(x^\star)$.
\end{theorem}
The proof of Theorem~\ref{th:evi} relies on using Lemma~\ref{lem:resolvent-LV} along with~\cite[Lemma 30]{durmus2018analysis}. Inspecting the proof of Theorem~\ref{th:evi}, one can see that any $\bar{\mu},\bar{y}$ can replace $\mu^\star,y^\star$\footnote{The proof does not rely on specific properties of the latter like being primal dual optimal.}. The situation is similar to primal dual algorithms in optimization~\cite{chambolle2016ergodic,drori2015simple} and Evolution Variational Inequalities in optimal transport~\cite{ambrosio2008gradient}.

%We claim without proof that $\mu^\star,y^\star$ in Inequality~\ref{eq:th} could be replaced by any $\bar{\mu} \in \cP_2^r(\sX),\bar{y} \in L^2$, similarly to Evolution Variational Inequalities in optimal transport, see~\cite[Th. 11.1.4]{ambrosio2008gradient}. \footnote{Indeed the proof of~\ref{th:evi} does not rely on specific properties of $\mu^\star, y^\star$, unlike Theorem~\ref{lem:DGbregman}.} In this case, the r.h.s. of~\eqref{eq:th} would involve the function $(\mu,y) \mapsto \sup_{\bar{\mu} \in \cP_2^r(\sX), \bar{y} \in L^2} \Lag(\mu,\bar{y}) - \Lag(\bar{\mu},y)$ instead of the duality gap $\DG(\mu^{k+1/2},y_\star^{k+1}) = \Lag(\mu^{k+1/2},y^\star) - \Lag(\mu^\star,y_\star^{k+1})$, similarly to primal dual algorithms in optimization~\cite{chambolle2016ergodic}.
%\asnote{Here. Dependence in the dimension.}
%\asnote{Relire, figures}
The next corollary is obtained by using $\DG(\mu^{k+1/2},y_\star^{k+1}) \geq 0$ (Theorem~\ref{lem:DGbregman}) and iterating~\eqref{eq:th}. 
% We first state a result in the general case $F$ convex smooth and $G \in \Gamma_0(\sX)$ in terms of duality gap. Similar results are obtained for Langevin algorithm when $F$ is not strongly convex in terms of Kullback Leibler divergence, see~\cite{dalalyan2019user,durmus2018analysis}. 

%Obtaining a result in terms of KL is hopeless for PSGLA because without further assumptions on $G$, because $\mu^k$ is not absolutely continuous w.r.t.\ $\mu^\star$ in general (hence $\KL(\mu^k|\mu^\star) = +\infty$). 
\begin{corollary}
\label{cor:rate}
Let Assumptions~\ref{as:smooth}--\ref{as:bound-var} hold true. If $\gamma \leq 1/L$, then
\begin{equation}
\label{eq:rateDG}
    \min_{j \in \{0,\ldots,k-1\}} \DG(\mu^{j+1/2},y_\star^{j+1}) \leq \tfrac{1}{2\gamma k}W^2(\mu^{0},\mu^\star) + \tfrac{\gamma}{2} C,
\end{equation}
\begin{equation}
\label{eq:ratedual}
    \min_{j \in \{1,\ldots,k\}} W^2(\nu^{j},\nu^\star)
    \leq \tfrac{1}{\gamma(\lambda_{G^\ast} + \gamma) k} W^2(\mu^{0},\mu^\star) + \tfrac{\gamma}{\lambda_{G^\ast} + \gamma} C.
\end{equation}
Finally, if $\lambda_F >0$, then 
\begin{equation}
\label{eq:ratescvx}
    W^2(\mu^{k},\mu^\star) \leq (1-\gamma \lambda_F)^k W^2(\mu^{0},\mu^\star) + \tfrac{\gamma}{\lambda_F} C.
\end{equation}

\end{corollary}
If $G$ is Lipschitz continuous (in particular if $G \equiv 0$), then our Assumptions hold true. Moreover, inequality~\eqref{eq:rateDG} recovers~\cite[Corollary 18]{durmus2018analysis} but with the duality gap instead of the KL divergence. Obtaining a result in terms of KL divergence is hopeless for PSGLA in general because the KL divergence is infinite; see the appendix. Connecting the convergence of the duality gap to zero to known modes of convergence is left for future work. Besides, obtaining an inequality like~\eqref{eq:ratedual} that holds when $F$ is just convex is rather not standard in the literature on Langevin algorithm, see~\cite{rolland2020double,zou2018stochastic}. Corollary~\ref{cor:rate} implies the following complexity results. Given $\varepsilon >0$, choosing $\gamma = \min(1/L,\varepsilon/C)$ and $k \geq \max(L/\varepsilon,C/\varepsilon^2)W^2(\mu^0,\mu^\star)$ in inequality~\eqref{eq:rateDG} leads to $\min_{j \in \{0,\ldots,k-1\}} \DG(\mu^{j+1/2},y_\star^{j+1}) \leq \varepsilon$. If $\lambda_{G^\ast} >0 $ (\textit{i.e.}, if $G$ is smooth), choosing $\gamma = \min(1/L,\frac{\lambda_{G^\ast}\varepsilon}{2C})$ and $k \geq \max(\frac{2L}{\lambda_{G^\ast}\varepsilon},\frac{4C}{\lambda_{G^\ast}^2\varepsilon^2})W^2(\mu^0,\mu^\star)$ in inequality~\eqref{eq:ratedual} leads to $\min_{j \in \{1,\ldots,k\}} W^2(\nu^{j},\nu^\star) \leq \varepsilon$. Finally, if $\lambda_{F} >0 $ (\textit{i.e.}, if $F$ is strongly convex), choosing $\gamma = \min(1/L,\frac{\lambda_{F}\varepsilon}{2C})$ and $k \geq \frac{1}{\gamma \lambda_F} \log(2 W^2(\mu^0,\mu^\star)/\varepsilon)$ \textit{i.e.},
\begin{equation}
\label{eq:k}
    k \geq \max\left(\tfrac{L}{\lambda_F},\tfrac{2 C}{\lambda_F^2 \varepsilon}\right)\log\left(\tfrac{2 W^2(\mu^0,\mu^\star)}{\varepsilon}\right), \quad C = \int_{\interior(D)} \|\nabla G(x)\|^2 d\mu^{\star}(x) + 2 (L d+\sigma_F^2)
\end{equation}
in inequality~\eqref{eq:ratescvx}, leads to $W^2(\mu^{k},\mu^\star) \leq \varepsilon$. \footnote{The dependence in $d$ of the factor $W^2(\mu^0,\mu^\star)$ can be explicited under further assumptions, see~\cite{rolland2020double,durmus2018analysis}.} In the case where $G$ is $M$-Lipschitz continuous, the complexity~\eqref{eq:k} improves~\cite[Corollary 22]{durmus2018analysis} since $\int_{\interior(D)} \|\nabla G(x)\|^2 d\mu^{\star}(x) \leq M^2$ .
\section{Conclusion}
We made a step towards theoretical understanding the properties of the Langevin algorithm in the case where the target distribution is not smooth and not fully supported. This case is known to be difficult to analyze and has many applications~\cite{brosse2017sampling,durmus2018efficient,bubeck2018sampling}. Our analysis improves and extends the state of the art.

Moreover, our approach is new. We developed a primal dual theory for a minimization problem over the Wasserstein space, which is of independent interest. A broader duality theory for minimization problems in the Wasserstein space would be of practical and theoretical interest.

\section{Acknowledgement}
We thank Laurent Condat for introducing us to the primal dual view of the proximal gradient algorithm in Hilbert spaces.

\section{Broader impact}
\label{sec:broader}
Our work contributes to the understanding of a sampling algorithm used in statistics. Our main results are of theoretical nature (convergence rates). Therefore, we do not see any immediate societal impact of our results.

\bibliographystyle{plain}

\begin{thebibliography}{10}

\bibitem{ambrosio2008gradient}
L.~Ambrosio, N.~Gigli, and G.~Savar{\'e}.
\newblock {\em Gradient {F}lows: {I}n {M}etric {S}paces and in the {S}pace of
  {P}robability {M}easures}.
\newblock Springer Science \& Business Media, 2008.

\bibitem{atchade2015moreau}
Y.~F Atchad{\'e}.
\newblock A {M}oreau-{Y}osida approximation scheme for a class of
  high-dimensional posterior distributions.
\newblock {\em arXiv preprint arXiv:1505.07072}, 2015.

\bibitem{atc-for-mou-17}
Y.~F Atchad{\'e}, G.~Fort, and E.~Moulines.
\newblock On perturbed proximal gradient algorithms.
\newblock {\em Journal of Machine Learning Research}, 18(1):310--342, 2017.

\bibitem{bau17}
H.~H. Bauschke and P.~L. Combettes.
\newblock {\em Convex Analysis and Monotone Operator Theory in Hilbert Spaces}.
\newblock Springer, New York, 2nd edition, 2017.

\bibitem{bernton2018langevin}
E.~Bernton.
\newblock {L}angevin {M}onte {C}arlo and {JKO} splitting.
\newblock In {\em Conference on Learning Theory}, pages 1777--1798, 2018.

\bibitem{bia-hac-16}
P.~Bianchi and W.~Hachem.
\newblock Dynamical behavior of a stochastic {F}orward-{B}ackward algorithm
  using random monotone operators.
\newblock {\em Journal of Optimization Theory and Applications},
  171(1):90--120, 2016.

\bibitem{bia-hac-sal-jca17}
P.~Bianchi, W.~Hachem, and A.~Salim.
\newblock A constant step {F}orward-{B}ackward algorithm involving random
  maximal monotone operators.
\newblock {\em Journal of Convex Analysis}, 26(2):397--436, 2019.

\bibitem{brazitikos2014geometry}
S.~Brazitikos, A.~Giannopoulos, P.~Valettas, and B.-H. Vritsiou.
\newblock {\em Geometry of isotropic convex bodies}, volume 196.
\newblock American Mathematical Soc., 2014.

\bibitem{brosse2017sampling}
N.~Brosse, A.~Durmus, E.~Moulines, and M.~Pereyra.
\newblock Sampling from a log-concave distribution with compact support with
  proximal {L}angevin {M}onte {C}arlo.
\newblock In {\em Conference on Learning Theory}, pages 319--342, 2017.

\bibitem{bubeck2018sampling}
S.~Bubeck, R.~Eldan, and J.~Lehec.
\newblock Sampling from a log-concave distribution with projected {L}angevin
  {M}onte {C}arlo.
\newblock {\em Discrete \& Computational Geometry}, 59(4):757--783, 2018.

\bibitem{chambolle2016ergodic}
A.~Chambolle and T.~Pock.
\newblock On the ergodic convergence rates of a first-order primal--dual
  algorithm.
\newblock {\em Mathematical Programming}, 159(1-2):253--287, 2016.

\bibitem{chatterji2019langevin}
N.~S Chatterji, J.~Diakonikolas, M.~I Jordan, and P.~L Bartlett.
\newblock Langevin monte carlo without smoothness.
\newblock In {\em International Conference on Artificial Intelligence and
  Statistics}, pages 1716--1726, 2020.

\bibitem{Chaux2007}
C.~Chaux, P.~L. Combettes, J.-C. Pesquet, and V.~R. Wajs.
\newblock A variational formulation for frame-based inverse problems.
\newblock {\em Inverse Problems}, 23(4):1495--1518, June 2007.

\bibitem{cheng2018convergence}
X.~Cheng and P.~L Bartlett.
\newblock Convergence of {L}angevin {MCMC} in {KL}-divergence.
\newblock In {\em Algorithmic Learning Theory}, pages 186--211, 2018.

\bibitem{com-pes-pafa16}
P.~L. Combettes and J.-C. Pesquet.
\newblock Stochastic approximations and perturbations in forward-backward
  splitting for monotone operators.
\newblock {\em Pure and Applied Functional Analysis}, 1(1):13--37, 2016.

\bibitem{condatreview}
L.~Condat, D.~Kitahara, A.~Contreras, and A.~Hirabayashi.
\newblock Proximal splitting algorithms: Overrelax them all!
\newblock {\em arXiv preprint arXiv:1912.00137}, 2019.

\bibitem{dalalyan2017theoretical}
A.~S Dalalyan.
\newblock Theoretical guarantees for approximate sampling from smooth and
  log-concave densities.
\newblock {\em Journal of the Royal Statistical Society: Series B (Statistical
  Methodology)}, 79(3):651--676, 2017.

\bibitem{drori2015simple}
Y.~Drori, S.~Sabach, and M.~Teboulle.
\newblock A simple algorithm for a class of nonsmooth convex--concave
  saddle-point problems.
\newblock {\em Operations Research Letters}, 43(2):209--214, 2015.

\bibitem{durmus2018analysis}
A.~Durmus, S.~Majewski, and B.~Miasojedow.
\newblock Analysis of {L}angevin {M}onte {C}arlo via convex optimization.
\newblock {\em Journal of Machine Learning Research}, 20(73):1--46, 2019.

\bibitem{durmus2017nonasymptotic}
A.~Durmus and E.~Moulines.
\newblock Nonasymptotic convergence analysis for the unadjusted {L}angevin
  algorithm.
\newblock {\em The Annals of Applied Probability}, 27(3):1551--1587, 2017.

\bibitem{durmus2018efficient}
A.~Durmus, E.~Moulines, and M.~Pereyra.
\newblock Efficient {B}ayesian computation by proximal {M}arkov {C}hain {M}onte
  {C}arlo: when {L}angevin meets {M}oreau.
\newblock {\em SIAM Journal on Imaging Sciences}, 11(1):473--506, 2018.

\bibitem{hsieh2018mirrored}
Y.-P. Hsieh, A.~Kavis, P.~Rolland, and V.~Cevher.
\newblock Mirrored {L}angevin dynamics.
\newblock In {\em Advances in Neural Information Processing Systems}, pages
  2878--2887, 2018.

\bibitem{ma2019there}
Y.-A. Ma, N.~Chatterji, X.~Cheng, N.~Flammarion, P.~L Bartlett, and M.~I
  Jordan.
\newblock Is there an analog of {N}esterov acceleration for {MCMC}?
\newblock {\em arXiv preprint arXiv:1902.00996}, 2019.

\bibitem{rockafellar1970convex}
R.~T Rockafellar.
\newblock {\em Convex analysis}, volume~28.
\newblock Princeton university press, 1970.

\bibitem{rolland2020double}
P.~Rolland, A.~Eftekhari, A.~Kavis, and V.~Cevher.
\newblock Double-loop unadjusted langevin algorithm.
\newblock In {\em International Conference on Machine Learning}, pages
  4326--4334, 2020.

\bibitem{rosasco2016stochastic}
L.~Rosasco, S.~Villa, and B.~C V{\~u}.
\newblock Stochastic forward--backward splitting for monotone inclusions.
\newblock {\em Journal of Optimization Theory and Applications},
  169(2):388--406, 2016.

\bibitem{salim2019stochastic}
A.~Salim, D.~Kovalev, and P.~Richt{\'a}rik.
\newblock Stochastic proximal langevin algorithm: Potential splitting and
  nonasymptotic rates.
\newblock In {\em Advances in Neural Information Processing Systems}, pages
  6649--6661, 2019.

\bibitem{schechtman2019passty}
S.~Schechtman, A.~Salim, and P.~Bianchi.
\newblock Passty {L}angevin.
\newblock In {\em Conference on Machine Learning (CAp)}, 2019.

\bibitem{tao2012topics}
Terence Tao.
\newblock {\em Topics in random matrix theory}, volume 132.
\newblock American Mathematical Soc., 2012.

\bibitem{toulis2015stable}
P.~Toulis, T.~Horel, and E.~M Airoldi.
\newblock The proximal {R}obbins-{M}onro method.
\newblock {\em Journal of the Royal Statistical Society: Series B (Statistical
  Methodology)}, 2020.
\newblock to appear.

\bibitem{vempala2019rapid}
S.~Vempala and A.~Wibisono.
\newblock Rapid convergence of the unadjusted {L}angevin algorithm:
  Isoperimetry suffices.
\newblock In {\em Advances in Neural Information Processing Systems}, pages
  8092--8104, 2019.

\bibitem{welling2011bayesian}
M.~Welling and Y.~W Teh.
\newblock Bayesian learning via stochastic gradient {L}angevin dynamics.
\newblock In {\em International Conference on Machine Learning}, pages
  681--688, 2011.

\bibitem{wibisono2018sampling}
A.~Wibisono.
\newblock Sampling as optimization in the space of measures: The {L}angevin
  dynamics as a composite optimization problem.
\newblock In {\em Conference on Learning Theory}, page 2093–3027, 2018.

\bibitem{wibisono2019proximal}
A.~Wibisono.
\newblock Proximal {L}angevin algorithm: Rapid convergence under isoperimetry.
\newblock {\em arXiv preprint arXiv:1911.01469}, 2019.

\bibitem{zou2018stochastic}
D.~Zou, P.~Xu, and Q.~Gu.
\newblock Stochastic variance-reduced {H}amilton {M}onte {C}arlo methods.
\newblock In {\em International Conference on Machine Learning}, pages
  6028--6037, 2018.

\end{thebibliography}
\newcommand{\noop}[1]{} \def\cprime{$'$} \def\cdprime{$''$} \def\cprime{$'$}

%\end{document}
\newpage
\appendix

\part*{Appendix}

\tableofcontents

\clearpage

\section{Numerical experiments}

In this section, we illustrate our results on PSGLA through numerical experiments. 
\paragraph{Sampling \textit{a posteriori}.} We consider a statistical framework where i.i.d.\ random vectors (data) $D_1,\ldots,D_n$ with distribution $\bP_{x^\star}$ are observed. We adopt a Bayesian strategy where we assume the distribution $\bP_{x^\star}$ to be indexed by a random vector $x^\star$ with values in $\sX$. Denote $\cL(\cdot,x^\star)$ the density of $\bP_{x^\star}$ (a.k.a. the likelihood function) w.r.t.\ some reference measure. Given a prior distribution for $x^\star$ with density $\pi$ w.r.t.\ $\Leb$, our goal is construct samples $x^1,\dots,x^k$ from the posterior distribution
\begin{equation}
\label{eq:posterior-distribution}
    \mu^\star(x | D_1,\ldots,D_n) \propto \pi(x)\prod_{i=1}^n \cL(D_i,x),
\end{equation}
in order \textit{e.g.} to estimate the mean \textit{a posteriori} via Monte Carlo approximations,
\begin{equation}
\label{eq:mean-a-posteriori}
    m^\star \eqdef \int x \mu^\star(x | D_1,\ldots,D_n)d\Leb(x) \simeq \frac{1}{k}\sum_{j = 1}^k x^j.
\end{equation}

\paragraph{Wishart distribution.} In the experiments, the Euclidean space $\sX$ is a space of $d\times d$ symmetric matrices and $\pi$ is the Wishart distribution defined by
\begin{equation}
\label{eq:Wishart}
    \pi(x) \propto |\det(x)|^{\frac{\nu-d-1}{2}}\exp \left(-\frac{\tr(V^{-1}x)}{2}\right)\indic_{\bR^{d\times d}_{++}}(x),
\end{equation}
where $\nu > d - 1$ and $V \in \bR^{d\times d}_{++}$ are parameters of the distribution. Note that $\pi(x)=0$ if $x$ is not a positive definite matrix. The mean of the Wishart distribution is equal to $\nu V$.
The Wishart distribution is widely used in Random matrix theory and applications, see~\cite{tao2012topics}. Indeed, the Wishart distribution is a conjugate prior to the Gaussian likelihood. More precisely,
assume that for every $D \in \bR^d$ and $x \in \bR^{d\times d}_{++}$,
\begin{equation}
\label{eq:Gaussian-likelihood}
    \cL(D,x) = \frac{1}{\sqrt{2\pi}^d}\exp\left(-\frac{1}{2}D^T x D\right)\sqrt{\det(x)},
\end{equation} 
is the density of a centered Gaussian distribution with precision matrix (\textit{i.e.}, inverse variance-covariance matrix) $x$. Then, if $\pi$ is Wishart with parameters $\nu$ and $V$ (\textit{i.e.}, $\pi$ is given by~\eqref{eq:Wishart}), then the posterior distribution $\mu^\star(\cdot | D_1,\ldots,D_n)$~\eqref{eq:posterior-distribution} is Wishart with parameters $\nu' = n + \nu$ and $V' = \left(I + \sum_{i=1}^n D_i D_i^T \right)^{-1}$:
\begin{equation}
\label{eq:conj}
    \mu^\star(x | D_1,\ldots,D_n) \propto |\det(x)|^{\frac{(\nu+n)-d-1}{2}}\exp \left(-\frac{\tr\left((V^{-1} + \sum_{i=1}^n D_i D_i^T)x\right)}{2}\right)\indic_{\bR^{d\times d}_{++}}(x).
\end{equation}
Moreover, the mean of the posterior distribution is equal to
\begin{equation}
\label{eq:mean-a-posteriori-2}
    m^\star = (n + \nu)\left(I + \sum_{i=1}^n D_i D_i^T \right)^{-1}.
\end{equation}

\paragraph{Setup.} We consider two \textit{a posterori} sampling problems.

First, we consider the task of learning the mean of the data. More precisely, $\bP_{x^\star}$ is a Gaussian distribution over $\bR$ with mean $x^\star$ and unit variance. We use the Wishart distribution $\pi$ with $V=I$ as the prior distribution. Note that in this one dimensional case, the Wishart distribution $\pi$ boils down to a Gamma distribution over $\bR$.

In other words, $\pi(x) \propto \exp(-G(x))$ and $\mu^\star(x | D_1,\ldots,D_n) \propto \exp(-G(x)-\sum_{i=1}^n f_i(x))$ where the functions $f_i,G \in \Gamma_0(\sX)$ are defined by
\begin{align*}
 G(x) &\eqdef -\frac{\nu-d-1}{2} \log |x| + \frac{x}{2} + \iota_{(0,+\infty)}(x),\\
 f_i(x) &\eqdef \frac{|x - D_i|^2}{2},
\end{align*} 
for every $i \in \{1,\ldots,n\}$. The data points $D_i$ are generated randomly using a Gaussian distribution. Note that $f_i$ is smooth and strongly convex. Moreover, $G$ is nonsmooth and the proximity operator of $G$ has a closed form thanks to recent results.\footnote{see \url{www.proximity-operator.net}} We consider $d=1$ in order to be able to represent the numerical results with histograms, see Figures~\ref{fig:simu-1d1}-\ref{fig:simu-1d2}.

\begin{figure}[ht!]
\[
  \begin{array}{cc}
 \includegraphics[width=.4\linewidth]{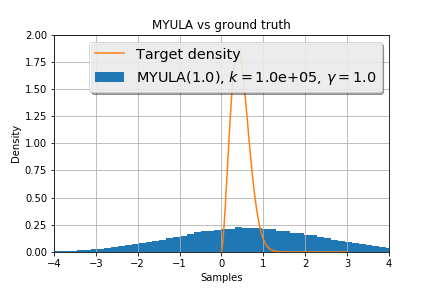} &
  \includegraphics[width=.4\linewidth]{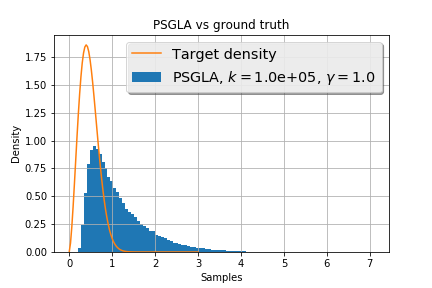}\\
    \includegraphics[width=.4\linewidth]{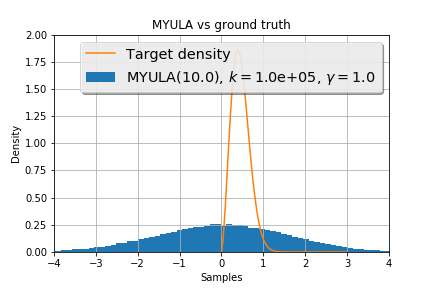} &
  \includegraphics[width=.4\linewidth]{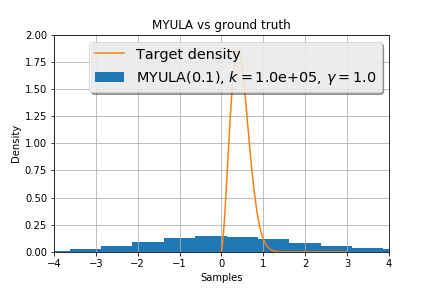}
  \end{array}
\]
    \caption{Histograms drawn by the $k$ iterates of PSGLA and MYULA($\lambda$), for various values of $\lambda$, compared to the target distribution $\mu^\star$. Case $d=1, \gamma = 1.0$.}
    \label{fig:simu-1d1}
\end{figure}

\begin{figure}[ht!]
\[
  \begin{array}{cc}
  \includegraphics[width=.4\linewidth]{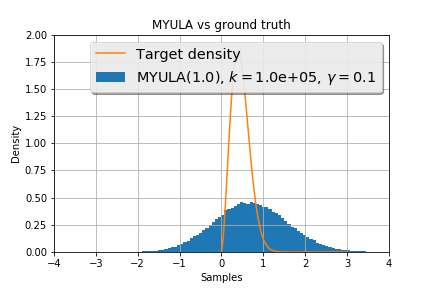} &
  \includegraphics[width=.4\linewidth]{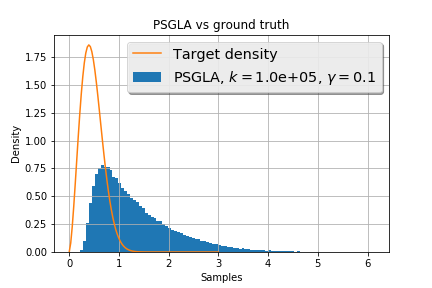}\\
    \includegraphics[width=.4\linewidth]{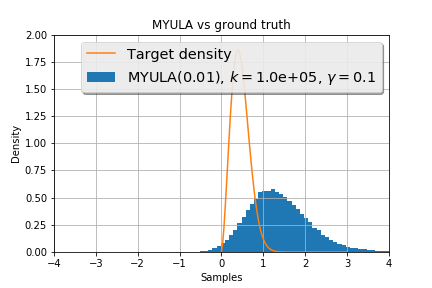} &
  \includegraphics[width=.4\linewidth]{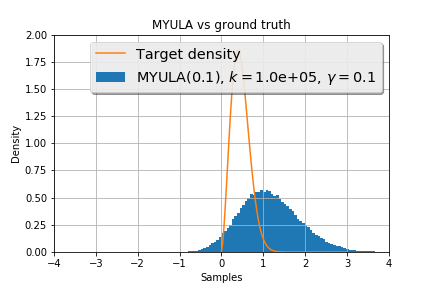}
  \end{array}
\]
    \caption{Histograms drawn by the $k$ iterates of PSGLA and MYULA($\lambda$), for various values of $\lambda$, compared to the target distribution $\mu^\star$. Case $d=1, \gamma = 0.1$.}
    \label{fig:simu-1d2}
\end{figure}

Then, we consider the task of learning the precision matrix of the data. More precisely, $\bP_{x^\star}$ is a centered Gaussian distribution over $\bR^d$ with precision matrix $x^\star$. We use the Wishart distribution $\pi$ with $V = I$ as the prior distribution. Since the prior distribution $\pi$~\eqref{eq:Wishart} is conjugate to the likelihood function $\cL(\cdot,x^\star)$~\eqref{eq:Gaussian-likelihood}, the posterior distribution $\mu^\star$ is given by~\eqref{eq:conj}. Thus, we can use $\mu^\star$ as a ground truth. In other words, $\mu^\star(x | D_1,\ldots,D_n) \propto \exp(-G(x)-F(x))$ where the functions $F,G \in \Gamma_0(\sX)$ are defined by
\begin{align*}
 G(x) &\eqdef -\frac{(\nu+n)-d-1}{2} \log |\det(x)| + \frac{\tr(x)}{2} + \iota_{\bR^{d\times d}_{++}}(x),\\
 F(x) &\eqdef \sum_{i = 1}^{n} \frac{\tr(D_i D_i^T x)}{2}.
\end{align*} 
The data points $D_i$ are generated randomly using a Gaussian distribution. Note that $F$ is smooth and convex, hence Assumptions~\ref{as:smooth}-\ref{as:int} are satisfied. Moreover, Assumptions~\ref{as:intgrad}-\ref{as:sobolev} are satisfied \textit{e.g.} if $n+\nu > d+3$. Finally, $G$ is nonsmooth and the proximity operator of $G$ has a closed form~\cite[Corollary 24.65]{bau17}. We consider several values of $d$: $ d = 1, d= 10$ and $d = 100$. The number of entries of the iterates is $d^2$ and, since the matrices are symmetric, the dimension of the sampling problem is slightly larger than $d^2 / 2$. Since the mean of $\mu^\star$ is known, we use it as a ground truth and perform a mean \textit{a posteriori} estimation using the estimators~\eqref{eq:mean-a-posteriori} constructed by PSGLA and MYULA. The convergence of the estimators is illustrated in Figures~\ref{fig:simu-d-10} and~\ref{fig:simu-d-100} for the cases $d = 10$ and $d = 100$. Moreover, in order to visualize better the multidimensional results of Figures~\ref{fig:simu-d-10} and~\ref{fig:simu-d-100}, we consider the \textit{same} sampling problem with $d=1$ and plot histograms to represent the results, see Figure~\ref{fig:simu-d-1}.

\begin{figure}[ht!]
\[
  \begin{array}{cc}
  \includegraphics[width=.4\linewidth]{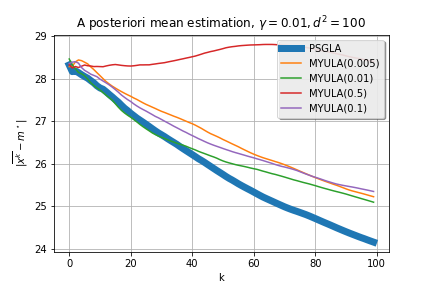} &
  \includegraphics[width=.4\linewidth]{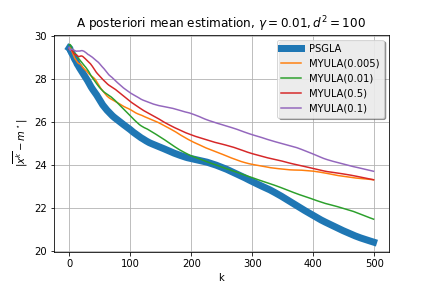}
  \\
    \includegraphics[width=.4\linewidth]{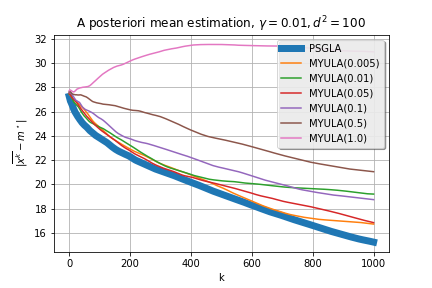} &
    \includegraphics[width=.4\linewidth]{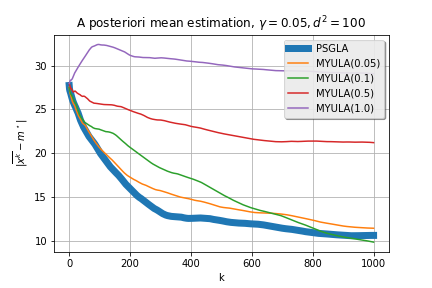}\\
  \includegraphics[width=.4\linewidth]{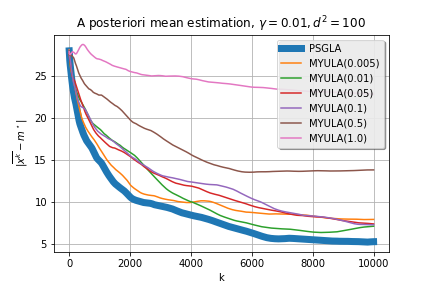} &
  \includegraphics[width=.4\linewidth]{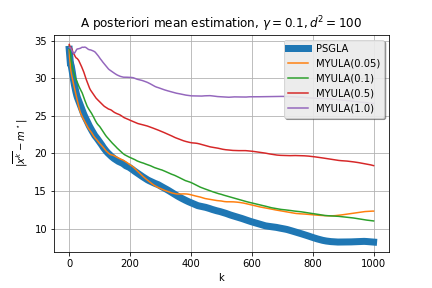}
  \end{array}
\]
    \caption{Frobenius distance between $m^\star$ and the mean \textit{a posteriori} estimators (ergodic means) constructed by PSGLA and MYULA($\lambda$), for various values of $\lambda, \gamma$ as a function of $k$ in the case $d = 10$.}
    \label{fig:simu-d-10}
\end{figure}

\begin{figure}[ht!]
\[
  \begin{array}{cc}
  \includegraphics[width=.4\linewidth]{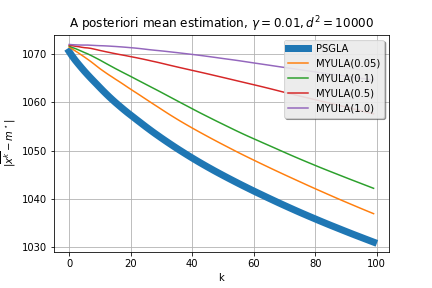} & \includegraphics[width=.4\linewidth]{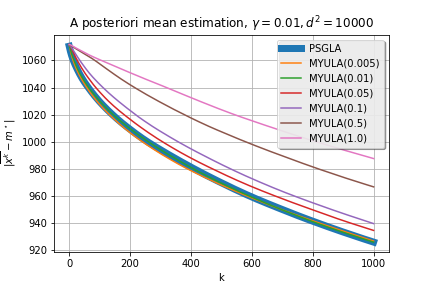}\\
    \includegraphics[width=.4\linewidth]{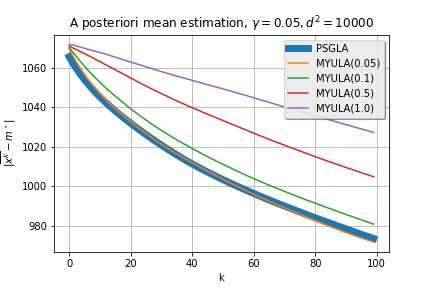} &
    \includegraphics[width=.4\linewidth]{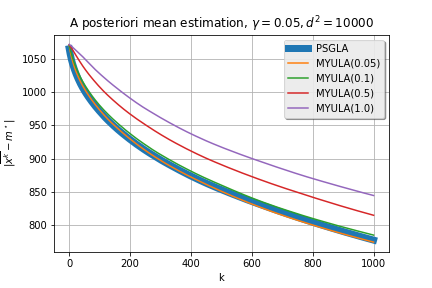}\\
    \includegraphics[width=.4\linewidth]{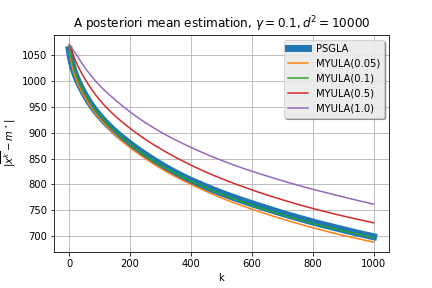}
   &
  \includegraphics[width=.4\linewidth]{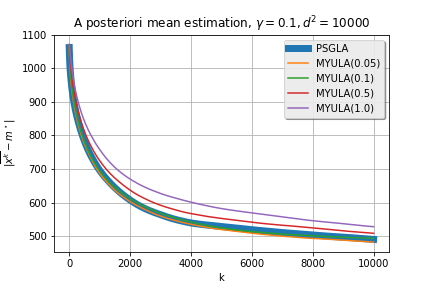}
  \end{array}
\]
    \caption{Frobenius distance between $m^\star$ and the mean \textit{a posteriori} estimators (ergodic means) constructed by PSGLA and MYULA($\lambda$), for various values of $\lambda, \gamma$ as a function of $k$ in the case $d = 100$.}
    \label{fig:simu-d-100}
\end{figure}

\begin{figure}[ht!]
\[
  \begin{array}{cc}
  \includegraphics[width=.4\linewidth]{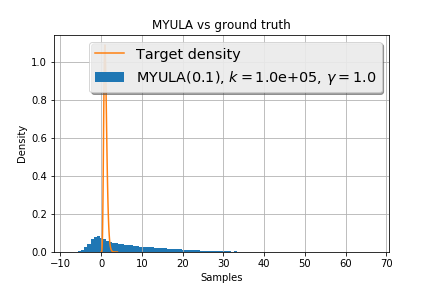} &
  \includegraphics[width=.4\linewidth]{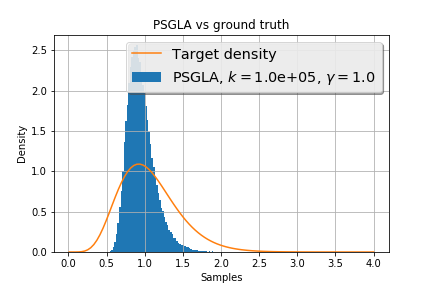}\\
    \includegraphics[width=.4\linewidth]{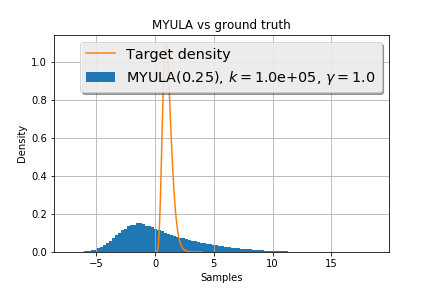} &
  \includegraphics[width=.4\linewidth]{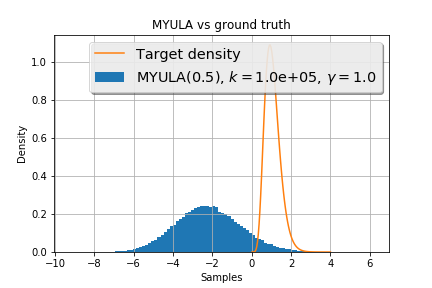}\\
    \includegraphics[width=.4\linewidth]{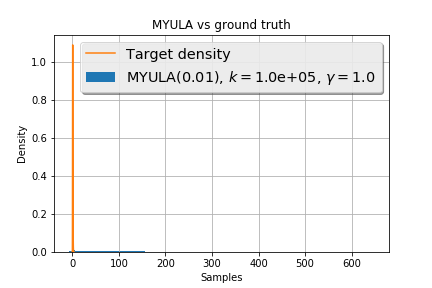} &
  \includegraphics[width=.4\linewidth]{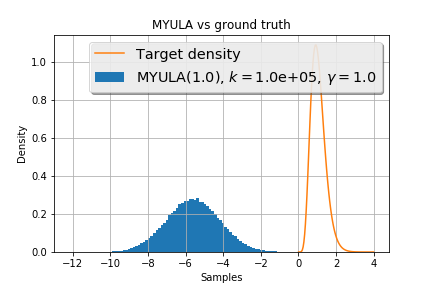}\\
  \end{array}
\]
    \caption{Histograms drawn by the $k$ iterates of PSGLA and MYULA($\lambda$), for various values of $\lambda$, compared to the target distribution $\mu^\star$. Case $d=1$, $\gamma = 1.0$.}
    \label{fig:simu-d-1}
\end{figure}

\paragraph{Algorithms.}
 We compare PSGLA to various versions of MYULA~\cite{durmus2018efficient}, parametrized by the smoothing parameter $\lambda >0$. We use the same learning rate $\gamma$ for the algorithms\footnote{In general, a Langevin algorithm becomes more precise and slower as $\gamma \to 0$.}. We denote these algorithms MYULA($\lambda$). Both MYULA and PSGLA compute one proximity operator $\prox_{\gamma G}$ and one gradient per iteration\footnote{We use a slight extension of MYULA allowing to use a stochastic gradient.}. They both require to sample one Gaussian random variable over the space of symmetric matrices at each iteration.

 \paragraph{Observations.}
 %\asnote{I would like to be kinder with MYULA.}
 
 Using Langevin algorithm to sample from distributions which are not fully supported is known to be a difficult task~\cite{bubeck2018sampling,brosse2017sampling,durmus2018efficient}. In Figures~\ref{fig:simu-1d1}-\ref{fig:simu-1d2} and~\ref{fig:simu-d-1}, we see that the shape of the histograms drawn by PSGLA are closer to the target distribution than the shape of the benchmarks histograms. This means that PSGLA converges faster than the benchmarks for this sampling task. This behavior was expected: PSGLA does not introduce extra bias by introducing a smoothing parameter $\lambda$.

 \textit{More importantly, we see that the iterates of PSGLA are always feasible i.e., they lie in the support of the target distribution, contrary to the benchmarks.}

 Finally, PSGLA is a proximal method, whereas the benchmarks are instances of the standard Langevin algorithm (applied to the smoothed problem depending on $\lambda$). In stochastic optimization, proximal methods are known to be more stable than gradient methods~\cite{toulis2015stable}. This phenomenon is observed here. For instance, the range of step sizes allowed by the benchmarks is controlled by the smoothing parameter $\lambda$, see~\cite{durmus2018efficient}. Using a step size too large for the benchmarks leads to a numerical instability that does not occur for PSGLA. 
 
 Figures~\ref{fig:simu-d-10} and~\ref{fig:simu-d-100} are multidimensional extensions of Figure~\ref{fig:simu-d-1}. Each figure in~\ref{fig:simu-d-10} and~\ref{fig:simu-d-100} corresponds to a new run. We plotted the convergence of the mean \textit{a posteriori} estimators (\textit{i.e.}, ergodic means) constructed by MYULA($\lambda$) and PSGLA. %Some values of $\lambda$ led to numerical instability and hence were not plotted. 
 
 \textit{In general, we see that PSGLA is as good as the MYULA($\lambda$) for the best value of $\lambda$, while conserving the feasibility of the iterates, and without having to select the value of $\lambda$.}

%\asnote{Explain parameters: d,n=10,N, Frobenius norm, ... The gradient of $F$ must be a symmetric matrix. Grad of $f_i$ (grad of log det), distribution of sampled data, Gaussian over the space of symmetric matrices.}

% \section{Acknowledgment}
% We thank Anonymous\asnote{Laurent Condat} for introducing us to the primal dual view of the proximal gradient algorithm.

\section{Postponed proofs}
\subsection{Proof of Lemma~\ref{lem:Vint}}
Using~\cite[Lemma 2.2.1]{brazitikos2014geometry}, there exist $A,B >0$, $\exp(-V(x)) \leq A\exp(-B\|x\|) \leq A$. The last inequality implies $\int \|x\|^2\exp(-V(x))dx < \infty$. Moreover, $-V(x) \leq \log(A) \leq C \eqdef \max(0,\log(A))$. 
% \begin{equation}
% \label{eq:integrable}
%     -V(x)\exp(-V(x)) \geq \left(\log(A) - B\|x\|\right)\exp(-V(x)).
% \end{equation}
Using that $u \mapsto u\exp(-u)$ is nonincreasing on $[1,+\infty)$.
\begin{align*}
    V(x)\exp(-V(x)) &= V(x)\exp(-V(x))\textbf{1}_{V(x) \leq \|x\|^2 + 1} + V(x)\exp(-V(x))\textbf{1}_{V(x) > \|x\|^2 + 1}\\
    &\leq \left(\|x\|^2 + 1\right)\exp(-V(x)) + \left(\|x\|^2 + 1\right)\exp\left(-(\|x\|^2 + 1)\right).
\end{align*}
Using $|V(x)| = V(x)\textbf{1}_{V(x)\geq0} - V(x)\textbf{1}_{V(x) < 0} \leq C + V(x)$,
$$|V(x)|\exp(-V(x)) \leq \left(\|x\|^2 + 1 + C\right)A\exp(-B\|x\|) + \left(\|x\|^2 + 1\right)\exp\left(-(\|x\|^2 + 1)\right).$$
We conclude using that the r.h.s. is integrable.

\subsection{Proof of Theorem~\ref{lem:DGbregman}}
The proof is divided in six parts, each part proving one claim. Denote $x = T_{\mu^\star}^\mu(x^\star)$. 

\paragraph{Part I.} First,
\begin{equation}
    \label{eq:Lagrangiannustar}
    \Lag(\mu,y^\star) = \cE_F(\mu) + \cH(\mu) - \cE_{G^\ast}(\nu^\star) + \bE\ps{x,y^\star},
\end{equation}
and
\begin{equation}
    \label{eq:Lagrangianmustar}
    \Lag(\mu^\star,y) = \cE_F(\mu^\star) + \cH(\mu^\star) - \cE_{G^\ast}(\nu) + \bE\ps{x^\star, y}.
\end{equation}
Therefore, the duality gap can be rewritten
\begin{align*}
    \DG(\mu,y) =& \cE_F(\mu) - \cE_F(\mu^\star) + \cH(\mu) - \cH(\mu^\star) +\cE_{G^\ast}(\nu) - \cE_{G^\ast}(\nu^\star)\\
    &+ \bE\ps{x, y^\star} - \bE\ps{x^\star, y}\\
    =& \cE_F(\mu) - \cE_F(\mu^\star) + \cH(\mu) - \cH(\mu^\star) +\cE_{G^\ast}(\nu) - \cE_{G^\ast}(\nu^\star)\\
    &+ \bE\ps{x - x^\star, y^\star} - \bE\ps{x^\star, y - y^\star}.
\end{align*}
Using~\eqref{eq:saddleW2}, $y^\star = -\nabla F(x^\star) -\partialb^0 \cH(\mu^\star)(x^\star)$ and $x^\star \in \partial G^\ast(y^\star)$.
\begin{align}
\label{eq:dg-sum-bregman}
    \DG(\mu,y) =& \cE_F(\mu) - \cE_F(\mu^\star) - \bE\ps{\nabla F(x^\star),x - x^\star}\nonumber\\
    &+ \cH(\mu) - \cH(\mu^\star) - \bE\ps{\partialb^0 \cH(\mu^\star)(x^\star),x - x^\star}\nonumber\\
    &+\cE_{G^\ast}(\nu) - \cE_{G^\ast}(\nu^\star) - \bE\ps{x^\star, y - y^\star}\nonumber\\
    =& \bE F(x) - \bE F(x^\star) - \bE\ps{\nabla F(x^\star),x - x^\star}\\
    &+ \cH(\mu) - \cH(\mu^\star) - \ps{\partialb^0 \cH(\mu^\star),T_{\mu^\star}^\mu - I}_{\mu^\star}\nonumber\\
    &+\bE G^\ast(y) - \bE G^\ast(y^\star) - \bE\ps{x^\star, y - y^\star},\nonumber
\end{align}
where the last equality comes from the transfer theorem. We get $\DG(\mu,y) \geq 0$ using the convexity of $F$, the convexity of $G^\ast$ and the geodesic convexity of $\cH$ (inequality~\eqref{eq:subdiff}).

\paragraph{Part II.} Since $\DG(\mu,y^\star) \geq 0$ and $\DG(\mu^\star,y) \geq 0$, 
\begin{equation}\Lag(\mu^\star,y) \leq \Lag(\mu^\star,y^\star) \leq \Lag(\mu,y^\star).\end{equation}
\paragraph{Part III.} Then,
\begin{align}
    \Lag(\mu,y) &= \cE_F(\mu) + \cH(\mu) + \bE\left(\ps{x,y} - G^\ast(y)\right)\label{eq:equal}\\
    &\leq \cE_F(\mu) + \cH(\mu) + \bE\left(\sup \ps{x,\cdot} - G^\ast\right)\label{eq:inequal}\\
    &\leq \cE_F(\mu) + \cH(\mu) + \bE G(x) \label{eq:last}\\
    &= \cF(\mu)\nonumber,
\end{align}
using~\cite[Proposition 13.15]{bau17}.

\paragraph{Part IV.} Using~\cite[Proposition 16.10]{bau17}, $\sup \ps{x,\cdot} - G^\ast = G(x)$, and $\ps{x,y} - G^\ast(y) = G(x)$ if and only if $y \in \partial G(x)$ (or $x \in \partial G^\ast(y)$).  
Taking $\mu = \mu^\star$ and $y = y^\star$ in~\eqref{eq:equal}, we have $x^\star = x \in \partial G^\ast(y^\star)$ and therefore $\Lag(\mu^\star,y^\star) = \cF(\mu^\star)$. 

\paragraph{Part V.} Assume that $\Lag(\mu^\star,\bar{y}) = \cF(\mu^\star)$. We shall prove that $\bar{y}=y^\star$ a.s. Since $\cF(\mu^\star) \leq \Lag(\mu^\star,\bar{y}),$ inequality~\eqref{eq:inequal} becomes an equality when $\mu = \mu^\star$ and $y = \bar{y}$. Therefore, $\ps{x^\star,\bar{y}} - G^\ast(\bar{y}) = \sup \ps{x^\star,\cdot} - G^\ast = G(x^\star)$ a.s., which implies $\bar{y} \in \partial G(x^\star) = \{\nabla G(x^\star)\} = \{y^\star\}$ a.s. 

\paragraph{Part VI.} Assume that $\Lag(\mu^\star,y^\star) = \Lag(\bar{\mu},y^\star)$ and that $F$ is strictly convex. We shall prove that $\mu^\star = \bar{\mu}$. We have $\DG(\bar{\mu},y^\star) = 0$, and, using~\eqref{eq:dg-sum-bregman}, 
\begin{equation}
\DG(\bar{\mu},y^\star) = \bE F(\bar{x}) - \bE F(x^\star) - \bE\ps{\nabla F(x^\star),\bar{x} - x^\star}
+ \cH(\bar{\mu}) - \cH(\mu^\star) - \ps{\partialb^0 \cH(\mu^\star),T_{\mu^\star}^{\bar{\mu}} - I}_{\mu^\star}, \end{equation}
where $\bar{x} = T_{\mu^\star}^{\bar{\mu}}(x^\star)$.
Using the convexity of $F$ and the geodesic convexity of $\cH$, $F(\bar{x}) - F(x^\star) - \ps{\nabla F(x^\star),\bar{x} - x^\star} = 0$ a.s. Using the strict convexity of $F$, $\bar{x} = x^\star$ a.s., therefore $\bar{\mu} = \mu^\star$. 
% using~\eqref{eq:last}, $$\cF(\mu^\star) = \Lag(\bar{\mu},y^\star) \leq \cF(\bar{\mu}).$$ 
% \asnote{Finish: use strict geo cvxity of entropy or F}

\subsection{Proof of Lemma~\ref{lem:resolvent-LV}}
If $x' = x -\gamma y'$ and $0 \in -x' + \partial G^\ast(y')$, we have  $y' \in \partial G(x')$ and $x' = \prox_{\gamma G}(x)$. Moreover, $0 \in -x + \gamma y' + \partial G^\ast(y')$ implies $y' = \prox_{G^\ast/\gamma}(x/\gamma)$. One can easily check that this is an equivalence. Denote $\|\cdot\|_P$ the semi-norm induced by $P$ on $\sX^2$ defined by $\|z\|_P = \|x\|$ for every $z = (x,y) \in \sX^2$, and $\ps{\cdot,\cdot}_P$ the semi-inner product associated. We have
\begin{equation}
\label{eq:funda-fb}
  \|z' - z^\star\|_P^2 %&= \|z - z^\star\|_P^2 +2\ps{z' - z,z - z^\star}_P +\|z' - z\|_P^2\nonumber\\
    = \|z - z^\star\|_P^2 +2\ps{z' - z,z' - z^\star}_P -\|z' - z\|_P^2.
 \end{equation}
 We now identify the terms. First, $\|z' - z^\star\|_P^2 = \|x' - x^\star\|^2$, $\|z - z^\star\|_P^2 = \|x - x^\star\|^2$ and $\|z' - z\|_P^2 = \|x'-x\|^2 = \gamma^2 \|y'\|^2$. Second, using $P(z' - z) \in -\gamma A(z')$, and the definition of $A$, there exists $g^\ast \in G^\ast(y')$ such that 
 \begin{equation}
 \ps{z' - z,z' - z^\star}_P = \ps{P(z' - z),z' - z^\star} = -\gamma \ps{y',x' - x^\star} + \gamma\ps{x',y' - y^\star} -\gamma\ps{g^\ast,y'-y^\star}.
 \end{equation}
  Hence 
  \begin{equation}\ps{z' - z,z' - z^\star}_P \leq \gamma\ps{y',x^\star} - \gamma \ps{x',y^\star} -\gamma \left(G^\ast(y') - G^\ast(y^\star) + \frac{\lambda_{G^\ast}}{2}\|y' - y^\star\|^2\right),
  \end{equation}
  using the strong convexity of $G^\ast$.
 Plugging into~\eqref{eq:funda-fb},
\begin{equation}
\label{eq:funda-pg}
    \|x' - x^\star\|^2 \leq \|x - x^\star\|^2 -2\gamma\left(G^\ast(y') - G^\ast(y^\star) + \frac{\lambda_{G^\ast}}{2}\|y' - y^\star\|^2 - \ps{y',x^\star}+ \ps{y^\star,x'}\right) -\gamma^2\|y'\|^2.
\end{equation}
Using $x' = x - \gamma y'$, 
\begin{equation}
    -2\gamma \ps{y^\star,x'} -\gamma^2\|y'\|^2 = -2\gamma \ps{y^\star,x} -2\gamma \ps{y^\star,- \gamma y'} -\gamma^2\|y'\|^2 = -2\gamma \ps{y^\star,x} -\gamma^2\|y'-y^\star\|^2 + \gamma^2\|y^\star\|^2.
\end{equation}
Plugging into~\eqref{eq:funda-pg} concludes the proof.

\subsection{Proof of Theorem~\ref{th:evi}}
We first recall a standard inequality of the stochastic gradient Langevin algorithm, see \textit{e.g.}~\cite[Lemma 30]{durmus2018analysis}.
\begin{lemma}[\cite{durmus2018analysis}]
\label{lem:durmus}
Let Assumptions~\ref{as:smooth},~\ref{as:int} and ~\ref{as:bound-var} hold true. Then, if $F$ is $\lambda_F$-strongly convex, for every $\gamma \leq 1/L$,
\begin{align}
W^2(\mu^{k+1/2},\mu^\star) \leq& (1-\gamma\lambda_F)W^2(\mu^k,\mu^\star) + 2\gamma^2 (L d +\sigma_F^2)\nonumber\\
&- 2\gamma\left(\cE_F(\mu^{k+1/2}) + \cH(\mu^{k+1/2}) - \cE_F(\mu^\star) - \cH(\mu^\star)\right).
\end{align}
\end{lemma}

We now prove Theorem~\ref{th:evi}.
% \begin{lemma}
% \label{lem:backward}
% Assume that assumption~\ref{as:smooth} holds.
% \begin{align}
%     W^2(\mu^{k+1},\mu^\star) \leq& W^2(\mu^{k+1/2},\mu^\star) \nonumber\\
%     &- \gamma^2 W^2(\nu^{k+1},\nu^\star)^2 \nonumber\\
%     &-2\gamma \left(-\cE_{G^\ast}(\nu^\star) + \bE \ps{ x_{\star}^{k+1/2}, y^\star}\right) \nonumber\\
%     &+2\gamma \left(-\cE_{G^\ast}(\nu^{k+1}) + \bE\ps{ x^\star,y_\star^{k+1}}\right) \nonumber\\
%     &+\gamma^2 \int \|y\|^2 d\nu^\star(y),
% \end{align}
% where $x_{\star}^{k+1/2} \eqdef T_{\mu^\star}^{\mu^{k+1/2}}(x^\star) \sim \mu^{k+1/2}$ and $y_{\star}^{k+1} = \prox_{G^\ast/\gamma}(x_{\star}^{k+1/2}/\gamma) \sim \nu^{k+1}$.
% \end{lemma}
The main tool for the proof is Lemma~\ref{lem:resolvent-LV}. Replace $x$ by $x^{k+1/2}_{\star} \sim \mu^{k+1/2}$ in~\eqref{eq:funda-lv}. Then $y' = y^{k+1}_{\star} \sim \nu^{k+1}$ and $\prox_{\gamma G}(x^{k+1/2}_{\star}) \sim \mu^{k+1}$. Therefore,
\begin{equation*}
W^2(\mu^{k+1},\mu^\star) \leq \bE(\|\prox_{\gamma G}(x^{k+1/2}_{\star}) - x^\star\|^2), \quad
W^2(\nu^{k+1},\nu^\star) \leq \bE(\|y^{k+1}_{\star}-y^\star\|^2).
\end{equation*}
 Consequently, taking expectation in~\eqref{eq:funda-lv} we get
\begin{align*}
    W^2(\mu^{k+1},\mu^\star) \leq& W^2(\mu^{k+1/2},\mu^\star) -\gamma(\lambda_{G^\ast} + \gamma) W^2(\nu^{k+1},\nu^\star) + \gamma^2 \int \|y\|^2 d\nu^\star(y)\\ &-2\gamma\left(\cE_{G^\ast}(\nu^{k+1}) - \cE_G^\ast(\nu^\star) - \bE\ps{y^{k+1}_{\star},x^\star}+ \bE\ps{y^\star,x^{k+1/2}_{\star}}\right).
\end{align*}
Combining with Lemma~\ref{lem:durmus}, we get the result.

\subsection{Proof of Corollary~\ref{cor:rate}}
From Theorem~\ref{th:evi},
\begin{align}
    \gamma(\lambda_{G^\ast} + \gamma) W^2(\nu^{j+1},\nu^\star) + 2\gamma \DG(\mu^{j+1/2},y_\star^{j+1})
    \leq& W^2(\mu^{j},\mu^\star) - W^2(\mu^{j+1},\mu^\star) +\gamma^2 C.
\end{align}
Summing over $j \in \{0,\ldots,k-1\}$,
\begin{align}
    &\gamma(\lambda_{G^\ast} + \gamma) \sum_{j=0}^{k-1} W^2(\nu^{j+1},\nu^\star) + 2\gamma \sum_{j=0}^{k-1} \DG(\mu^{j+1/2},y_\star^{j+1})
    \\
    \leq& W^2(\mu^{0},\mu^\star) - W^2(\mu^{k},\mu^\star) +k \gamma^2 C.
\end{align}
Therefore, 
\begin{align}
    &\gamma(\lambda_{G^\ast} + \gamma) k \min_{j \in \{0,\ldots,k-1\}} W^2(\nu^{j+1},\nu^\star) + 2\gamma k \min_{j \in \{0,\ldots,k-1\}} \DG(\mu^{j+1/2},y_\star^{j+1})
    \\
    \leq& W^2(\mu^{0},\mu^\star) + k \gamma^2 C,
\end{align}
which implies
\begin{equation}
    \min_{j \in \{0,\ldots,k-1\}} \DG(\mu^{j+1/2},y_\star^{j+1})
    \leq \frac{1}{2\gamma k}W^2(\mu^{0},\mu^\star) + \gamma \frac{C}{2},
\end{equation}
and, 
\begin{equation}
    \min_{j \in \{1,\ldots,k\}} W^2(\nu^{j},\nu^\star)
    \leq \frac{1}{\gamma(\lambda_{G^\ast} + \gamma) k} W^2(\mu^{0},\mu^\star) + \frac{\gamma}{\lambda_{G^\ast} + \gamma} C.
\end{equation}
Moreover, if $\lambda_F >0$, Theorem~\ref{th:evi} implies
\begin{equation}
    W^2(\mu^{k+1},\mu^\star)
    \leq (1-\gamma\lambda_F)W^2(\mu^{k},\mu^\star) + \gamma^2 C.
\end{equation}
Iterating, we obtain
\begin{equation}
    W^2(\mu^{k},\mu^\star)
    \leq (1-\gamma\lambda_F)^k W^2(\mu^{0},\mu^\star) + \gamma \frac{C}{\lambda_F}.
\end{equation}
\section{Further intuition on PSGLA}

\subsection{Stochastic gradient descent interpretation of the stochastic gradient Langevin algorithm}
As mentionned in the introduction, Langevin algorithm can be interpreted as a gradient descent algorithm in the space $\cP_2(\sX)$ to minimize $\KL(\cdot|\mu^\star)$, see \textit{e.g.}~\cite{durmus2018analysis}. More precisely, consider the case where $G \equiv 0$ and denote $\mu^k$ the distribution of $x^k$. Then PSGLA boils down to the stochastic gradient Langevin algorithm (\textit{i.e.}, PSGLA without proximal step) and satisfy the following inequality (Lemma~\ref{lem:durmus})
\begin{equation}
\label{eq:LangevinSGD}
W^2(\mu^{k+1},\mu^\star) \leq (1-\gamma\lambda_F)W^2(\mu^k,\mu^\star)
- 2\gamma\left(\cF(\mu^{k+1}) - \cF(\mu^\star)\right) + 2\gamma^2 (L d +\sigma_F^2),
\end{equation}
if $F$ is $L$-smooth, $\lambda_F$-strongly convex and $\gamma \leq 1/L$. The last inequality is similar to a standard inequality used in the analysis of SGD. More precisely, the analysis of SGD often relies on an inequality similar to~\eqref{eq:LangevinSGD}, by replacing the Wasserstein distance by the Euclidean distance and $\cF$ by the objective function to be minimized by SGD (note that $L d +\sigma_F^2$ is a constant). Therefore, unrolling the recursion~\eqref{eq:LangevinSGD} (which is the standard way to obtain convergence rates for SGD) leads to the complexity $\cO(1/\varepsilon^2)$ in terms of objective gap $\cF(\mu) - \cF(\mu^\star)$. Using~\eqref{eq:KL-F}, recall that the objective gap is the KL divergence. 

In this paper, we considered the case $G \neq 0$. One can obtain an inequality similar to~\eqref{eq:LangevinSGD} for PSGLA if $G$ is Lipschitz continuous, see~\cite{durmus2018analysis,salim2019stochastic}. However, for a general $G \in \Gamma_0(\sX)$, it is hopeless. Indeed, $\cF(\mu^{k+1}) - \cF(\mu^\star)= +\infty$ in general because $\cH(\mu^{k+1}) = +\infty$ since $\mu^{k+1}$ is not absolutely continuous w.r.t.\ $\Leb$ (\textit{e.g.} when the proximal step is a projection). Moreover, $\cF(\mu^{k+1/2}) - \cF(\mu^\star)= +\infty$ in general because $\cE_G(\mu^{k+1/2}) = +\infty$ since $\mu^{k+1/2}$ is not supported by $\dom(G)$ ($\supp(\mu^{k+1/2}) = \sX$ because of the Gaussian noise). Therefore, one cannot obtain a rate in terms of KL divergence (\textit{i.e.}, objective gap) for PSGLA in general, since the KL divergence is equal to $+\infty$. 

On order to overcome this difficulty, we assumed~\ref{as:sobolev} and adopted a primal dual interpretation of PSGLA where PSGLA is seen as a Forward Backward algorithm involving monotone operators. We obtained an inequality similar to~\eqref{eq:LangevinSGD}, but with the duality gap instead of the objective gap, and we proved that the duality gap is nonnegative. 

%We used Lemma~\ref{lem:durmus} in the proof of Theorem~\ref{th:evi} only to connect $\mu^k$ to $\mu^{k+1/2}$.

%Before proceeding further, note that PSGLA can be viewed as a stochastic gradient Langevin algorithm step, followed by a proximal step.  In the context of the stochastic gradient Langevin algorithm of Inequality~\eqref{eq:LangevinSGD}, $\mu^{k+1}$ represents the distribution of the iterate after the stochastic gradient Langevin step, and $\cF(\mu)$ is equal to $\cE_F(\mu) + \cH(\mu)$ since $G \equiv 0$. Therefore, Inequality~\eqref{eq:LangevinSGD} can be used to connect $\mu^k$ and $\mu^{k+1/2}$ of PSGLA. More precisely, by making the change of variable $\mu^{k+1} \leftarrow \mu^{k+1/2}$ in Inequality~\eqref{eq:LangevinSGD}, we have
%\begin{align}
%\label{eq:lem30}
%W^2(\mu^{k+1/2},\mu^\star) \leq& (1-\gamma\lambda_F)W^2(\mu^k,\mu^\star) + 2\gamma^2 (L d +\sigma_F^2)\nonumber\\
%&- 2\gamma\left(\cE_F(\mu^{k+1/2}) + \cH(\mu^{k+1/2}) - \cE_F(\mu^\star) - \cH(\mu^\star)\right),
%\end{align}
%where $\mu^k, \mu^{k+1/2}$ are defined by PSGLA.

\subsection{Primal dual interpretation of the proximal gradient algorithm}
The approach of this paper can also be used to interpret the proximal gradient algorithm as a primal dual algorithm. 

Consider the minimization problem
\begin{equation}
\label{eq:min-suppl}
    \min_{x \in \sX} F(x) + G(x).
\end{equation}
To solve Problem~\eqref{eq:min-suppl}, the proximal gradient algorithm is written
\begin{equation}
\label{eq:pg}
    x^{k+1} = \prox_{\gamma G} \left(x^k - \gamma \nabla F(x^k)\right).
\end{equation}
The proximal gradient algorithm can be seen as a primal dual algorithm for Problem~\eqref{eq:min-suppl}~\cite{rockafellar1970convex}.
% \begin{equation}
%     \label{eq:min-suppl}
%         \min_{y \in \sX} F^*(-y) + G^*(y).
%     \end{equation}
    
Indeed, a solution $x^\star$ to Problem~\eqref{eq:min-suppl} satisfies $0 \in \nabla F(x^\star) + \partial G(x^\star)$. Consider the dual variable $y^\star \in \partial G(x^\star)$ such that $0 = \nabla F(x^\star) + y^\star$. Since, $y^\star \in \partial G(x^\star)$, $0 \in - x^\star + \partial G^\ast(y^\star)$ using $\partial G^\ast = (\partial G)^{-1}.$ Finally,
\begin{equation}
\label{eq:inc-suppl}
\begin{bmatrix} 0 \\ 0\end{bmatrix} \in \begin{bmatrix} \nabla F(x^\star) & + y^\star \\ - x^\star& + \partial G^\ast(y^\star)\end{bmatrix}.
\end{equation}
Consider the set valued maps
$$ B : (x,y) \mapsto \begin{bmatrix} \nabla F(x)\\ 0\end{bmatrix}, $$ and
$$ A : (x,y) \mapsto \begin{bmatrix} &   y \\ - x& +\partial G^\ast(y)\end{bmatrix},$$ where we used vector notation. The maps $A$ and $B$ are maximal monotone operators (note that $B$ is the gradient of $(x,y) \mapsto F(x)$ and $A$ was used in Section~\ref{sec:alg}). Inclusion~\eqref{eq:inc-suppl} can be rewritten as 
\begin{equation}
\label{eq:zero-suppl}
    0 \in (A+B)(x^\star,y^\star).
\end{equation}
In order to solve~\eqref{eq:zero-suppl}, one can apply the Forward Backward algorithm
\begin{equation}
\label{eq:FB-suppl}
    P(x^{k+1/2} - x^k) = -\gamma B(x^k), \quad P(x^{k+1} - x^{k+1/2}) \in -\gamma A(x^{k+1}),
\end{equation}
for a well chosen $P \in \bR_{++}^{d \times d}$. As above, we take $P : (x,y) \mapsto x$. Although the matrix $P$ is only semi-definite positive, we showed in Lemma~\eqref{lem:resolvent-LV} that $x^{k+1} = \prox_{\gamma G}(x^{k+1/2})$. Hence, the primal dual Forward Backward algorithm~\eqref{eq:FB-suppl} is equivalent to the proximal gradient algorithm~\eqref{eq:pg}.

Moreover, the proof technique used for Theorem~\ref{th:evi} can be adapted to analyze the proximal gradient algorithm as a primal dual algorithm. The complexity result obtained for the proximal gradient algorithm with this approach is suboptimal. However, the derivation of this complexity result sheds some light on PSGLA.

First, using the (strong) convexity of $F$,
\begin{align*}
        \|x^{k+1/2} - x^\star\|^2 &= \|x^k - x^\star\|^2 + \gamma^2 \|\nabla F(x^k)\|^2 - 2\gamma\ps{\nabla F(x^k), x^k - x^\star}\\
         &\leq (1-\gamma \lambda_F)\|x^k - x^\star\|^2 + \gamma^2 \|\nabla F(x^k)\|^2 -2\gamma \left (F(x^k) - F(x^\star) \right)\\
         &\leq (1-\gamma \lambda_F)\|x^k - x^\star\|^2 + \gamma^2 \|\nabla F(x^k)\|^2 -2\gamma \left (F(x^{k+1/2}) - F(x^\star) \right)\\ &\phantom{=}-2\gamma \left (F(x^{k}) - F(x^{k+1/2}) \right).
    \end{align*}
Using the smoothness of $F$,
\begin{equation*}
    F(x^{k+1/2}) - F(x^k) \leq \ps{\nabla F(x^k),x^{k+1/2} - x^k} + \frac{L}{2}\|x^{k+1/2} - x^k\|^2 = -\gamma \left(1 - \frac{\gamma L}{2} \right)\|\nabla F(x^k)\|^2.
\end{equation*}
Therefore, 
\begin{equation}
    \label{eq:forward-suppl}
    \|x^{k+1/2} - x^\star\|^2 \leq  \|x^k - x^\star\|^2 - \gamma^2\left(1- \gamma L\right) \|\nabla F(x^k)\|^2 -2\gamma \left (F(x^{k+1/2}) - F(x^\star) \right).
\end{equation}
Inequality~\eqref{eq:forward-suppl} is analogue to Lemma~\ref{lem:durmus}.
Moreover, using Lemma~\ref{lem:resolvent-LV},
\begin{align*}
    \|x^{k+1} - x^\star\|^2 \leq& \|x^{k+1/2} - x^\star\|^2 \\&-2\gamma\left(G^\ast(y^{k+1}) - G^\ast(y^\star) - \ps{y^{k+1},x^\star}+ \ps{y^\star,x^{k+1/2}}\right)\\ &-\gamma(\lambda_{G^\ast} + \gamma)\|y^{k+1} - y^\star\|^2 + \gamma^2 \|y^\star\|^2,
\end{align*}
where $y^{k+1} = \prox_{G^\ast/\gamma}(x^k/\gamma)$.
Summing the two last inequality, and using $\gamma \leq 1/L$,
\begin{align}
\label{eq:evi-pdpg}
    \|x^{k+1} - x^\star\|^2
    \leq& (1-\gamma\lambda_F)\|x^{k} - x^\star\|^2  -\gamma(\lambda_{G^\ast} + \gamma) \|y^{k+1} - y^\star\|^2 \nonumber\\
    &-2\gamma \left(\Lag(x^{k+1/2},y^\star) - \Lag(x^\star,y^{k+1})\right) + \gamma^2 \|y^\star\|^2,
\end{align}
where $\Lag(x,y) = F(x) - G^\ast(y) + \ps{x,y}$ is the Lagrangian function and $\Lag(x^{k+1/2},y^\star) - \Lag(x^\star,y^{k+1})$ is the duality gap. The last inequality is similar to Theorem~\ref{th:evi}.
\begin{remark}
With slight modifications of the derivations above, one can get the better result
\begin{align*}
    \|x^{k+1} - x^\star\|^2
    \leq& (1-\gamma\lambda_F)\|x^{k} - x^\star\|^2  -2\gamma \left(\Lag(x^{k+1},y^\star) - \Lag(x^\star,y^{k+1})\right).
\end{align*}
However, the proof technique would not adapt to Langevin algorithm.
\end{remark}
\begin{remark}
Similarly to the result of Theorem~\ref{th:evi}, $x^\star,y^\star$ can be replaced by any $\bar{x},\bar{y}$. The proof technique does not use specific properties of $x^\star,y^\star$, as being primal dual optimal. Primal dual optimality of $x^\star,y^\star$ is only needed to prove that the duality gap is nonnegative.
\end{remark}

\section{Generalization to a stochastic three operators splitting}

In order to cover more applications, for instance involving \textit{several Lipschitz proximable terms} in the potential, we quickly generalize the results of Section~\ref{sec:res}. Our primal dual framework can be plugged to the results of~\cite{salim2019stochastic} instead of~\cite{durmus2018analysis}, leading to an extension of Section~\ref{sec:res}.

Consider the task of sampling from $\mu^\star \propto \exp(-V)$, where 
\begin{equation}
    V(x) = \bE(f(x,\xi)) + \bE(r(x,\xi)) + G(x).
\end{equation}
We assume the following. 
\begin{assumption}
\label{as:lip}
For every $x \in \sX$, $r(x,\xi)$ is integrable and $R(x) \eqdef \bE_\xi(r(x,\xi))$. Moreover, $r(\cdot,\xi) \in \Gamma_0(\sX)$ a.s. Finally, there exists $M \geq 0$ such that for every $x \in \sX$, $\bE_\xi(\|\partial^0 r(x,\xi)\|^2) \leq M^2$.
\end{assumption}
Assumption~\ref{as:lip} holds \textit{e.g.} if $r(x,\xi)$ is $\ell(\xi)$-Lipschitz continuous and $\bE_\xi(\ell^2(\xi)) < \infty$, since $\bE_\xi(\|\nabla^0 r(x,\xi)\|^2) \leq \bE_\xi(\ell^2(\xi))$. By replacing $F$ by $F+R$ in Section~\ref{sec:pd-opt}, 
Theorem~\ref{lem:DGbregman} still hold with the Lagrangian function
\begin{equation}
    \label{eq:Lagrangianspla}
    \Lag(\mu,y) \eqdef \cE_F(\mu)+\cE_R(\mu) + \cH(\mu) - \cE_{G^\ast}(\nu) + \bE \ps{x, y},
\end{equation}
where $x = T_{\mu^\star}^\mu(x^\star)$. Note that $\dom(R) = \sX$, hence $R$ is differentiable a.e.\ using~\cite[Theorem 25.5]{rockafellar1970convex}.
%\asnote{$\lambda$-cvx}
In order to sample from $\mu^\star$, the stochastic proximal Langevin algorithm (SPLA)~\cite{salim2019stochastic} is written
\begin{equation}
    \label{eq:spla}
    x^{k+1} = \prox_{\gamma G}\left(\prox_{\gamma r(\cdot,\xi)}\left(x^k - \gamma \nabla_x f(x^k,\xi^{k+1}) + \sqrt{2\gamma} W^{k+1}\right)\right).
\end{equation}
SPLA recovers PSGLA by taking $R \equiv 0$. Moreover, SPLA is analyzed in~\cite{salim2019stochastic} only in the case where $G$ also satisfies Assumption~\ref{as:lip} (and hence $\dom(G) = \sX$) which is stronger than assuming~\ref{as:sobolev}. In this section, we denote 
\begin{equation}
    \label{eq:xk1/2spla}
    x^{k+1/2} \eqdef \prox_{\gamma r(\cdot,\xi)}\left(x^k - \gamma \nabla_x f(x^k,\xi^{k+1}) + \sqrt{2\gamma} W^{k+1}\right),
\end{equation}
and $\mu^{k+1/2}$ its distribution.
We shall prove the following extension of Theorem~\ref{th:evi} assuming that $G$ satisfies the Assumption~\ref{as:sobolev}. This extension of Theorem~\ref{th:evi} leads to an extension of Corollary~\ref{cor:rate} providing complexity results for SPLA similar to PSGLA.

\begin{theorem}
\label{th:evi2}
Let Assumptions~\ref{as:smooth},~\ref{as:int}, ~\ref{as:intgrad},~\ref{as:bound-var} and~\ref{as:lip} hold true. If $F$ is $\lambda_F$-strongly convex and $G$ is $1/\lambda_{G^\ast}$-smooth, then for every $\gamma \leq 1/L$,
\begin{align}
\label{eq:thspla}
    W^2(\mu^{k+1},\mu^\star)
    \leq& (1-\gamma\lambda_F)W^2(\mu^{k},\mu^\star)  -\gamma(\lambda_{G^\ast} + \gamma) W^2(\nu^{k+1},\nu^\star) \nonumber\\
    &-2\gamma \left(\Lag(\mu^{k+1/2},y^\star) - \Lag(\mu^\star,y_\star^{k+1})\right) + \gamma^2 C,
\end{align}
where $C \eqdef \int_{\interior(D)} \|\nabla G(x)\|^2 d\mu^{\star}(x) + 2 (L d+\sigma_F^2 + M^2)$ and
$
y_{\star}^{k+1} \eqdef \prox_{G^\ast/\gamma}(x^{k+1/2}_{\star}/\gamma) \sim \nu^{k+1},
$ where $x^{k+1/2}_{\star} \eqdef T_{\mu^\star}^{\mu^{k+1/2}}(x^\star)$.
\end{theorem}

Before proving Theorem~\ref{th:evi2}, we recall the following consequence of~\cite[Theorem 1]{salim2019stochastic}, which generalizes Lemma~\ref{lem:durmus}.
\begin{lemma}[\cite{salim2019stochastic}]
\label{lem:salim}
Let Assumptions~\ref{as:smooth},~\ref{as:int},~\ref{as:bound-var} and~\ref{as:lip} hold true. Then, if $F$ is $\lambda_F$-strongly convex, for every $\gamma \leq 1/L$,
\begin{align}
W^2(\mu^{k+1/2},\mu^\star) \leq& (1-\gamma\lambda_F)W^2(\mu^k,\mu^\star) + 2\gamma^2 (L d +\sigma_F^2 +M^2)\nonumber\\
&- 2\gamma\left(\cE_F(\mu^{k+1/2}) + \cE_R(\mu^{k+1/2}) + \cH(\mu^{k+1/2}) - \cE_F(\mu^\star) - \cE_R(\mu^{\star}) - \cH(\mu^\star)\right).
\end{align}
\end{lemma}
\begin{proof}
Apply~\cite[Theorem 1]{salim2019stochastic} by taking $G_2 \equiv \ldots \equiv G_n \equiv 0$ and noting that the KL term in~\cite[Equation 3]{salim2019stochastic} is equal to $\left(\cE_F(\mu^{k+1/2}) + \cE_R(\mu^{k+1/2}) + \cH(\mu^{k+1/2}) - \cE_F(\mu^\star) - \cE_R(\mu^{\star}) - \cH(\mu^\star)\right)$ using our notations, see Equation~\eqref{eq:KL-F}.
\end{proof}

We now prove Theorem~\ref{th:evi2}, similarly to Theorem~\ref{th:evi}.

The main tool for the proof is Lemma~\ref{lem:resolvent-LV}. Replace $x$ by $x^{k+1/2}_{\star} \sim \mu^{k+1/2}$ in~\eqref{eq:funda-lv}. Then $y' = y^{k+1}_{\star} \sim \nu^{k+1}$ and $\prox_{\gamma G}(x^{k+1/2}_{\star}) \sim \mu^{k+1}$. Therefore,
\begin{equation*}
W^2(\mu^{k+1},\mu^\star) \leq \bE(\|\prox_{\gamma G}(x^{k+1/2}_{\star}) - x^\star\|^2), \quad
W^2(\nu^{k+1},\nu^\star) \leq \bE(\|y^{k+1}_{\star}-y^\star\|^2).
\end{equation*}
 Consequently, taking expectation in~\eqref{eq:funda-lv} we get
\begin{align*}
    W^2(\mu^{k+1},\mu^\star) \leq& W^2(\mu^{k+1/2},\mu^\star) -\gamma(\lambda_{G^\ast} + \gamma) W^2(\nu^{k+1},\nu^\star) + \gamma^2 \int \|y\|^2 d\nu^\star(y)\\ &-2\gamma\left(\cE_{G^\ast}(\nu^{k+1}) - \cE_G^\ast(\nu^\star) - \bE\ps{y^{k+1}_{\star},x^\star}+ \bE\ps{y^\star,x^{k+1/2}_{\star}}\right).
\end{align*}
Combining with Lemma~\ref{lem:salim}, we get the result.

\end{document}